\documentclass[a4paper,11pt]{article}
\usepackage{microtype}
\usepackage{graphicx}
\usepackage{booktabs}
\usepackage{amsfonts}
\usepackage{amsmath}
\usepackage{amssymb}
\usepackage{amsthm}
\usepackage{subcaption}
\usepackage{xcolor}
\usepackage{bm}
\usepackage{rotating}
\usepackage[colorlinks,citecolor=blue]{hyperref}
\usepackage[capitalize,noabbrev]{cleveref}
\usepackage{algorithm2e}
\RestyleAlgo{ruled}
\usepackage{fancyhdr}
\usepackage{listings}
\usepackage{afterpage}
\usepackage{natbib}

\usepackage{dsfont}

\usepackage[textsize=tiny]{todonotes}

\newcommand{\ey}[2][]{\todo[color=blue!20,#1]{{\bf Ey:} #2}}

\title{Fairness Constraints for Set-Valued Classification}

\hypersetup{
    colorlinks=true,
    linkcolor=blue,
    filecolor=magenta,      
    urlcolor=cyan,
    pdftitle={Overleaf Example},
    pdfpagemode=FullScreen,
}

\newtheorem{theorem}{Theorem}[section]
\newtheorem{corollary}[theorem]{Corollary}
\newtheorem{lemma}[theorem]{Lemma}
\theoremstyle{definition}
\newtheorem{remark}[theorem]{Remark}
\theoremstyle{definition}
\newtheorem{definition}[theorem]{Definition}
\newtheorem{proposition}[theorem]{Proposition}
\newtheorem{assumption}{Assumption}

\crefname{assumption}{Assumption}{Assumptions}
\crefname{algocfline}{Algorithm}{Algorithms}
\crefname{lemma}{Lemma}{Lemmas}
\crefname{Definition}{definition}{definitions}
\crefname{theorem}{Theorem}{Theorems}
\crefname{remark}{Remark}{Remarks}
\crefname{algorithm}{Algorithm}{Algorithms}
\crefname{lstlisting}{Listing}{Listings}

\newcommand{\size}[1]{\left\lvert#1\right\rvert}

\newcommand{\E}[2][]{\ifthenelse{\equal{#2}{}}{\mathbb{E}_{#1}}{\mathbb{E}_{#1}\left[ #2 \right]}}
\newcommand{\X}{\mathcal{X}}
\newcommand{\Scal}{\mathcal{S}}
\newcommand{\proba}[2][]{\ifthenelse{\equal{#2}{}}{\mathbb{P}_{#1}}{\mathbb{P}_{#1}\left(#2\right)}}
\newcommand{\probahat}[2][]{\ifthenelse{\equal{#2}{}}{\widehat{\mathbb{P}}_{#1}}{\widehat{\mathbb{P}}_{#1}\left(#2\right)}}

\newcommand{\1}[1]{\mathds{1}_{\left\{#1\right\}}}
\newcommand{\Rcal}{\mathcal{R}}

\newcommand{\norm}[2][]{\left\lvert\left\lvert#2\right\rvert\right\rvert_{#1}}

\newcommand{\sumover}[2][]{\overset{#1}{\underset{#2}{\sum}}}

\newcommand{\argmin}[1]{\underset{#1}{\operatorname{arg}\!\operatorname{min}}\;}

\title{Set to Be Fair: Demographic Parity Constraints for Set-Valued Classification}

\author{Eyal H.~Cohen$^{(1)}$,  Christophe Denis$^{(2)}$, Mohamed Hebiri$^{(3)}$}

\begin{document}
\date{}
\maketitle
\begin{center}
$(1)$ Sorbonne Université, LPSM, UMR 8001\\
${(2)}$ Université Paris 1 Panthéon-Sorbonne, SAMM\\ 
$(3)$ Université Gustave Eiffel, LAMA, UMR 8050
\end{center}

\begin{abstract}

Set-valued classification is used in multiclass settings where confusion between classes can occur and lead to misleading predictions. However, its application may amplify discriminatory bias motivating the development of set-valued approaches under fairness constraints. 
In this paper, we address the problem of set-valued classification under demographic parity and expected size constraints.
We propose two complementary strategies: an oracle-based method that minimizes classification risk while satisfying both constraints, and a computationally efficient proxy that prioritizes constraint satisfaction.
For both strategies, we derive closed-form expressions for the (optimal) fair set-valued classifiers and use these to build plug-in, data-driven procedures for empirical predictions. We establish distribution-free convergence rates for violations of the size and fairness constraints for both methods, and under mild assumptions we also provide excess-risk bounds for the oracle-based approach. Empirical results demonstrate the effectiveness of both strategies and highlight the efficiency of our proxy method.


\end{abstract}


\section{Introduction}\label{Introduction}
Set-valued classifiers are powerful tools for handling ambiguity between class labels in multiclass classification problems. Their popularity grew with the advent of conformal prediction~\citep{Vovk} and has become increasingly important in large-scale settings.
Numerous set-valued frameworks now coexist, each offering different trade-offs and applications~\citep{Denis_Hebiri_2017,Chzhen_Denis_Hebiri_Lorieul_2021,Sadinle_Lei_Wasserman_2019}.
In parallel, the rapid expansion of machine learning and deep learning in critical and sensitive domains such as medicine~\citep{Celard_Iglesias_Sorribes-Fdez_Romero_Vieira_Borrajo_2023}, hiring~\citep{Chen_2023}, criminal justice~\citep{Taylor_2023}, and banking~\citep{Sadok_Sakka_El_Maknouzi_2022,Amato_Osterrieder_Machado_2024}, has made algorithmic fairness a central concern in the statistical and machine learning communities~\citep{hardt2016equality,agarwal18a,Chzhen_Denis_Hebiri_Oneto_Pontil19,Paulus_Kent_2020,Hobson_Yesberg_Bradford_Jackson_2023,Yang_Soltan_Eyre_Clifton_2023,Chen_Wang_Williamson_Chen_Lipkova_Lu_Sahai_Mahmood_2023,Cameron_Cheong_Spitale_Gunes_2024}.
The main issue addressed by algorithmic fairness is the mitigation of learned biases and discrimination arising from sensitive attributes such as gender, ethnicity, or socioeconomic status.
A wide range of methods attempt to implement fairness through pre-, in-, or post-processing, targeting either exact fairness or approximate fairness (also called $\epsilon$-fairness), the latter allowing for improved trade-offs with predictive performance~\citep{zemel2013learning,lum2016statistical,calders2009building,feldman2015certifying,zafar2017fairness,barocas-hardt-narayanan,Chzhen_Denis_Hebiri_Oneto_Pontil19,jiang2019wasserstein,gordaliza2019obtaining,hardt2016equality,dwork2012}. See~\citep{Alves_Bernier_Couceiro_Makhlouf_Palamidessi_Zhioua_2023} for a recent review.
Approximate fairness offers a flexible way to control fairness while limiting the accuracy drop. However, it requires a predefined level of unfairness, which can be difficult to interpret and calibrate in practice.

In this paper, we consider a set-valued classification problem involving multiple classes and a sensitive attribute. The goal is to build a classifier that outputs a subset of classes while ensuring fairness -- in a sense to be specified later -- with respect to the sensitive attribute and controlling the average size of the output to limit information disclosure.
We focus on exact fairness, but allow a compromise on the interpretability of the output in order to reduce the classification risk. Specifically, we adopt the framework of set-valued classification~\citep{Lapin2016,Denis_Hebiri_2017,Sadinle_Lei_Wasserman_2019}, which includes ideas related to the conformal prediction setting~\citep{Vovk}. In this framework, the classifier may output multiple candidate labels, and the misclassification risk is naturally defined as the probability that the true label is not included in the predicted set.
A key advantage of set-valued classifiers under an expected size constraint is that, by allowing larger outputs in ambiguous cases, one can reduce the overall misclassification risk. However, it has been observed that in some applications where set-valued classification is particularly appropriate -- such as image classification -- biases in the data may lead to systematic misclassification~\citep{Besse2018CanEA}. This highlights the need to develop fair set-valued classifiers.
In this work, we incorporate demographic parity (DP) as a fairness constraint to ensure that the classifier does not discriminate based on the sensitive attribute.

        
    
    \subsection{Related Work}

        
Fairness in classification has been extensively studied under various criteria such as demographic parity~\citep{Calders_Kamiran_Pechenizkiy_2009}, equalized odds, and equal opportunity~\citep{hardt2016equality}. Two scenarios are typically considered: awareness and unawareness~\citep{agarwal18a,Chzhen_Denis_Hebiri_Oneto_Pontil19,Wang_Zhang_Zhu_2022,gaucher23a} -- whether we have access to the sensitive attribute at prediction time or not. With exact fairness being the concept of calibrating the algorithms to completely remove biases with respect to a given sensitive attribute, a relaxed version, known as approximate fairness or $\epsilon$-fairness allows for a trade-off between accuracy and fairness~\citep{agarwal18a,Denis_Elie_Hebiri_Hu_2024}.  While appealing from a performance standpoint, $\epsilon$-fairness is often less interpretable, as it relies on empirically chosen thresholds for acceptable unfairness. To improve interpretability, $\alpha$-fairness has been proposed~\citep{Chzhen_Schreuder_2022}, which seeks predictions that are at most $\alpha$ times as unfair as an unconstrained baseline, providing a clearer and more intuitive fairness guarantee.
        
Conformal prediction offers a natural framework for set-valued classification by providing calibrated prediction sets with coverage guarantees~\citep{Gibbs_Cherian_Candes_2024, Vovk}. Recent work has extended this framework to incorporate fairness constraints, such as adaptively selecting features and equalizing coverage across groups~\citep{Zhou_Sesia_2024}, or by combining conformal prediction with quantile regression and fairness adjustments~\citep{Romano_Patterson_Candes,Liu_Ding_Yu_Liu_Kong_Jiang}.
More broadly, set-valued predictors have been widely used to address class ambiguity in multiclass problems (see~\citep{Chzhen_Denis_Hebiri_Lorieul_2021} for a review) but has not been explored from the fairness perspective yet. 
        
Our work focuses on set-valued classification under fairness and size constraints. We provide an explicit solution of the fair set-valued classifier along with theoretical guarantees on constraint violations and excess risk. We also show that, while using a post-processing approach, the constraint violations guarantees are independent of the quality of the underlying estimators.

\subsection{Main contributions}
\label{ssec:main contribution}
        
        
        Our work focuses on the set-valued classification problem and the demographic fairness constraint. Our main contributions are the following: {\bf i)} we extend the notion of demographic parity to the set-valued classification setting;
        {\bf ii)} we exhibit a closed-form expression of the optimal fair set-valued classifier under an expected size constraint and deduce from its expression a data-driven procedure based on the plug-in principle. A key feature of the method is its post-processing nature: any preliminary estimator of the conditional probabilities can be used to build a fair set-valued classifier with the prescribed size, using {\it unlabeled} data only, making it attractive in practice.
        We provide theoretical controls on the risk, the unfairness and the size of the proposed set-valued classifier. Notably, both guarantees on the constraints are distribution-free.
        {\bf iii)} we propose a computationally efficient alternative to the optimal approach that avoids the need for solvers. Although not optimal, this proxy satisfies the same constraint violation guarantees, making it a practical alternative.
        {\bf iv)} we conduct numerical comparisons on both synthetic and real data, demonstrating the relevance of both approaches in practice.

    \subsection{Paper Outline}\label{ssec:outline}
The rest of the paper is organized as follow. In Section~\ref{sec:General Framework}, we formally introduce the problem of fair set-valued classification under a size constraint, along with a formal characterization of the optimal fair set-valued classifier with constrained size.
Section~\ref{sec:Data Driven Procedure} presents a plug-in approach that mimics this optimal set-valued classifier by solving a constrained optimization problem. Section~\ref{sec:twoStepProc} introduces a computationally simpler two-step procedure based on post-processing an unfair classifier to enforce fairness. We detail its statistical guarantees and compare both methods from a computational perspective. Section~\ref{sec:Experiments and Results} provides empirical results on synthetic and real-world data to evaluate the trade-offs between statistical accuracy, fairness, and computational cost. We conclude and discuss future directions in Section~\ref{sec:conclusion}.

\section{General Framework}
\label{sec:General Framework}

In this section, we start presenting in Section~\ref{subsec:statSetting} the general setting as well as the main definitions relevant to our problem. We then derive the optimal set-valued classifier and discuss its properties in Section~\ref{subsec:OptPredict}.
    
\subsection{Statistical setting}
\label{subsec:statSetting}
We begin with some useful notation. Let $K\geq 2$ be an integer and write $[K]$ to denote the set $\{1,\ldots,K\}$. Let $(X,S,Y) \in \X \times \Scal \times [K]$  be a random tuple with distribution $\proba[]{}$, respectively denoting by $X$ the covariates, $S$ the sensitive attribute, and $Y$ the class label.
A set-valued classifier is a function mapping $\X \times \Scal$ to the power set of classes $2^{[K]}$. Let $\bm{\Gamma}$ denote the collection of all set-valued classifiers.
For any $\Gamma \in \bm{\Gamma}$, two quantities are of interest: the expected size $\mathcal{T}(\Gamma) = \mathbb{E}\left[\left|\Gamma(X,S)\right|\right]$ and the risk $R(\Gamma) = \mathbb{P}\left(Y \notin \Gamma(X,S)\right)$. These two objectives are typically in tension: larger sets tend to reduce the risk but increase the size.
%
%
For every $(x,s,k)$ in $\X \times \Scal \times [K]$, we denote by $p_k(x,s) = \proba{Y=k | X=x, S=s}$ the conditional class probabilities. The marginal distribution of the sensitive attribute $S$ is denoted by $\pi_s := \proba{S = s}$ for each $s \in \Scal$.
A central tool in our analysis is the use of cumulative distribution functions (cdf) and their general inverses. For each $k\in [K]$ and $s\in \Scal$, we denote by $F_k$ (respectively  $F_{k,s}$) the cdf of $p_k(X,S)$ under the distribution $\mathbb{P}_{(X,S)}$ of $(X,S)$ (respectively the conditional distribution $\mathbb{P}_{X|S=s}$ of $X$ given $S=s$). Moreover, for any real-valued random variable $U$, we define $\overline{F}_{U} = 1 - F_{U}$. Finally, we introduce the function $G$ defined by $G(t) := \sumover[K]{k=1} \overline{F}_k(t)$ for $t \in \mathbb{R}$ and denote by $G^{-1}$ its generalized inverse.

\paragraph*{DP-fair set-valued classifier.}
We address the fairness problem within the Demographic Parity~(DP) framework adapted to the set-valued setting. This leads to the following definition:
\begin{definition}[DP-constraint]
\label{def:set-valued fairness}
    A set-valued classifier $\Gamma \in \bm{\Gamma}$ is said to be DP-fair if, for all $ k \in [K]$ and $ s \in \Scal$
    \begin{equation*}
        \proba[X \lvert S = s]{k \in \Gamma(X,s)} = \proba[X,S]{k \in \Gamma(X,S)} \enspace .
    \end{equation*}
    We denote by $\bm{\Gamma}_{\rm fair}$ the set of all classifiers satisfying the DP constraint.
\end{definition}
This definition is a direct extension of the notion of DP in classification \citep{Calders_Kamiran_Pechenizkiy_2009} to the set-valued setting. Our goal here is to build a set-valued classifier that minimizes the risk under the DP constraint and that has a bounded expected size. More formally, for a fixed limiting size $\beta > 0$, we aim to solve the following constrained optimization problem:
\begin{equation}
\label{eq:eqFairOpt}
\Gamma_{\beta}^* \in \arg\min \left\{R(\Gamma) \;\; : \;\; \Gamma \in \bm{\Gamma}_{{\rm fair}}, \;\;  \mathcal{T}(\Gamma) \leq \beta   \right\}   \enspace .
\end{equation}

\subsection{Optimal Predictor}
\label{subsec:OptPredict}
One convenient way to get a closed form expression of the optimal predictor $\Gamma_{\beta}^* $ is to lies under the following assumption.
\begin{assumption}[Continuity]
\label{assum:continuous}
        For each $k \in[K]$ and $s \in \mathcal{S}$, the cdf $F_{k,s}$ is continuous.
    \end{assumption}
\begin{theorem}
    \label{thm:optimal set-valued classifier} 
    Suppose Assumption~\ref{assum:continuous} is verified. 
    Then the $\beta$-specific oracle $\Gamma^*_\beta$ is:
    \begin{equation*}
        \Gamma^*_{\beta}(x,s) = \left\{ k \in [K]: p_k(x,s) \geq \lambda^* + \frac{\gamma^*_{k,s}}{\pi_s} \right\} \enspace,
    \end{equation*}
    with $\gamma^*_{k,s} = \alpha^*_{k,s} - \pi_s \sum_s \alpha^*_{k,s}$ and
     $\lambda^*$ and $\alpha^* = \left(\alpha_{k,s}\right)_{k \in [K], s\in \mathcal{S}}$ are the Lagrangian multiplier that are characterized as
     \footnotesize{
     \begin{equation}
     \label{eq:set-valued fair lagrangian}
        (\lambda^*,\alpha^*) \in \argmin{\substack{(\lambda, \alpha) \in \mathbb{R}^{K|\mathcal{S}|+1} \\ \lambda \geq 0 }} \sum_{k = 1}^K\sum_{s \in \mathcal{S}} \E[X|S=s]{\left( \pi_s \left(p_k(X, s) - \lambda + \sum_{s \in \mathcal{S}} \alpha_{k,s} \right) - \alpha_{k,s}\right)_{+}} +\lambda\beta \enspace ,
    \end{equation}
    }
    where $(\cdot)_+$ stands for the positive part.
\end{theorem}
The above result shows that, under Assumption~\ref{assum:continuous}, the optimal predictor can be characterized as a thresholding rule applied to the conditional probabilities $p_k$. This threshold is composed of two components: the first, $\lambda^*$, is a Lagrange multiplier associated with the expected size constraint, and is therefore responsible for calibrating the average size of the predictor $\Gamma_{\beta}^*$. The second component adjusts $\lambda^*$ in a class- and group-specific manner to enforce the fairness constraint.
Notably, this characterization extends the one derived in~\cite{Denis_Hebiri_2017}, where only the expected size constraint is considered. In their setting, the threshold involves a single parameter that does not depend neither on the class-label nor on the sensitive feature.

An important issue that remains is the resolution of the optimization problem in Equation~\eqref{eq:set-valued fair lagrangian}. The Lagrange multipliers obtained are not unique: the fairness-related parameters $\alpha^*$ can be shifted by a common constant without affecting the resulting classifier $\Gamma_\beta^*$. To address this, and in light of the definition of the optimal fairness parameter $\gamma^*$, which satisfies $\sum_{s \in \mathcal{S}} \gamma_{k,s}^* = 0$, the optimization problem can be reparameterized as follows:
\begin{equation}
\label{eq:set-valued fair lagrangian2}
(\lambda^*,\gamma^*) \in \arg\min_{\substack{(\lambda, \gamma) \in \mathbb{R}^{K|\mathcal{S}|+1} \\ \lambda \geq 0 \\ \sum_{s \in \mathcal{S}}\gamma_{k,s} = 0 }} \sum_{k = 1}^K\sum_{s \in \mathcal{S}} \E[X|S=s]{\left( \pi_s \left(p_k(X, s) - \lambda \right) - \gamma_{k,s}\right)_{+}} +\lambda\beta \enspace .
\end{equation}
{While this reparameterization does not ensure the uniqueness of the pair $(\lambda^*,\gamma^*) $, it simplifies the optimization landscape and the construction of the oracle predictor $\Gamma^*_\beta$. In particular, it allows for easier verification of key properties of the objective function, such as coercivity. We now state several properties of the $\beta$-specific oracle $\Gamma^*_\beta$, which will facilitate the analysis in the following sections.}

\paragraph*{Risk measure.} 
The next result  provides an important characterization of the optimal predictor.
\begin{proposition}
\label{prop:charctOptimRisk}
Let $\Gamma^*_{\beta}$ be the optimal predictor. Under Assumption~\ref{assum:continuous} the following holds
\begin{enumerate}
    \item[(i)] $\mathcal{T}\left(\Gamma_{\beta}^*\right) = \beta$; \hspace{10.5cm} \textbf{\textit{[Size validity]}} 
    \item[(ii)] for each $(k,s) \in [K] \times \mathcal{S}$
    \begin{equation*}
    \proba[X \lvert S = s]{k \in \Gamma_{\beta}^*(X,s)} = \mathbb{P}_{X,S}\left(k \in \Gamma_{\beta}^*(X,S)\right)  \enspace; \textbf{ \hspace{3.35cm} \textit{[DP-fair validity]}}
    \end{equation*}
    \item[(iii)] $\Gamma^*_{\beta} \in \arg\min_{\Gamma \in \bm{\Gamma}} \mathcal{R}_{\lambda^*, \gamma^*} \left(\Gamma\right)$, with
    \begin{equation*}
    \mathcal{R}_{\lambda^*, \gamma^*} \left(\Gamma\right) = R(\Gamma) + \lambda^* \left(\E[X]{\left\lvert \Gamma(X,S) \right\rvert} - \beta\right) 
        + \sumover[K]{k=1} \sumover{s \in \Scal} \gamma^*_{k,s} \proba[X \lvert S = s]{k \in \Gamma(X,s)}\enspace.
    \end{equation*}
\end{enumerate}
\end{proposition}
The above proposition shows that the optimal predictor achieves the prescribed expected size $\beta$ and can be characterized as the minimizer, over all set-valued classifiers, of the Lagrangian objective $\mathcal{R}_{\lambda^*, \gamma^*}$. In particular, this highlights that $\mathcal{R}_{\lambda^*, \gamma^*}$ serves as a relevant surrogate risk in our framework, as it naturally balances three competing objectives: classification accuracy, expected size, and fairness.
Moreover, this characterization allows us to derive a closed-form expression for the excess risk of any classifier $\Gamma \in \bm{\Gamma}$ relative to the optimal fair predictor.
\begin{corollary}
\label{coro:excessRisk}
Let $(\lambda^*,\gamma^*)$ be a solution of~\eqref{eq:set-valued fair lagrangian2}. Then for each $\Gamma \in \bm{\Gamma}$, we have that
\begin{equation*}
\mathcal{R}_{\lambda^*, \gamma^*}\left(\Gamma\right)-\mathcal{R}_{\lambda^*, \gamma^*}\left(\Gamma_{\beta}^*\right) = 
\sumover[K]{k=1} \ \sumover{s \in \Scal} \  \E[X|S=s]{\1{k \in \Gamma(X,s) \Delta \Gamma^*_{\beta}(X,s)} \left\lvert \pi_s\left(p_k(X,s) - \lambda^*\right) - \gamma^*_{k,s} \right\rvert} \enspace,
\end{equation*}
where $\Delta$ stands for the symmetric difference of two sets.
\end{corollary}
A direct consequence of the above result is that, under Assumption~\ref{assum:continuous}, the optimal predictor $\Gamma_{\beta}^*$ is \emph{a.s.} unique -- this follows from the expression of the excess risk, which involves the symmetric difference between $\Gamma$ and $\Gamma^*_{\beta}$ on the right-hand side. 
In particular, if $\widetilde{\Gamma}$ is any solution to the minimization problem in Equation~\eqref{eq:eqFairOpt}, then $\widetilde{\Gamma} = \Gamma^*_{\beta}$ \emph{a.s.}

\paragraph*{On the uniqueness of the optimal predictor.}
We have shown that the optimal predictor is almost surely unique. Under a more structural assumption, we can further establish the uniqueness of the parameters $(\lambda^*, \gamma^*)$ from~\eqref{eq:set-valued fair lagrangian2} that characterize $\Gamma_{\beta}^*$. To that end, we strengthen Assumption~\ref{assum:continuous} with the following condition:
\begin{assumption}[Positive density]\label{assum:positive density}
    For each $s \in \mathcal{S}$, the random variables $p_{k}(X,S)$ admit a strictly positive and continuous density {\it w.r.t.} $\mathbb{P}_{X|S=s}$.
\end{assumption}
This assumption ensures that both $F_{k,s}$ and $F_k$ are bijective. In particular, we obtain the following result:
\begin{proposition}
\label{prop:uniqueness}
Suppose that Assumption~\ref{assum:positive density} holds. Then:
\begin{enumerate}
    \item[$(i)$] the optimal parameters $(\lambda^*, \gamma^*)$ are unique;
    \item[$(ii)$] the optimal predictor $\Gamma_{\beta}^*$ admits the following unique parametrization:
    \begin{align*}
        \Gamma^*_{\beta}(x,s) = \left\{ k \in [K] : p_k(x,s) \geq \bar{F}^{-1}_{k,s} \left( \beta_k^* \right) \right\}, 
        \qquad \text{with} \quad \beta_k^* = \mathbb{P}\left(k \in \Gamma_{\beta}^*(X,S)\right)\enspace .
    \end{align*}
\end{enumerate}
\end{proposition}

{Under the positive density assumption on $p_k(X,S)$, the expression of the threshold in Theorem~\ref{thm:optimal set-valued classifier} simplifies. In particular, we have 
$\bar{F}^{-1}_{k,s} \left( \beta_k^* \right) =  \lambda^* + \frac{\gamma^*_{k,s}}{\pi_s}$.
In Remark~\ref{rk:uniqueness}, we leverage this expression to highlight key optimality properties.
}

\section{Data-Driven Procedure}
\label{sec:Data Driven Procedure}
This section is devoted to the presentation of our estimation procedure and the analysis of its theoretical guarantees. We first describe the overall methodology in Section~\ref{subsec:procedure}, and then establish rates of convergence for the proposed algorithm in Section~\ref{subsec:RatesofConv}.

\subsection{Procedure}
\label{subsec:procedure}
Our estimation procedure aims to recover the $\beta$-specific DP-fair set-valued classifier $\Gamma_{\beta}^*$ introduced in Theorem~\ref{thm:optimal set-valued classifier}, following the plug-in principle. The overall strategy consists in estimating the unknown components involved in the expression of $\Gamma_{\beta}^*$. Notably, some of these components do not depend on the full data distribution $\proba[]{}$, which enables a semi-supervised estimation approach -- reminiscent of the approach proposed in~\cite{Denis_Elie_Hebiri_Hu_2024}.

More formally, let $n, N > 1$ be two integers.
We assume access to two independent datasets: a first labeled dataset denoted by $\mathcal{D}_n = \left\{(X_i, S_i, Y_i), i = 1, \ldots, n \right\}$, and a second unlabeled dataset $\mathcal{D}_N = \left\{(X_i, S_i), i = n+1, \ldots, n+N \right\}$.
Based on $\mathcal{D}_n$, for each $k \in [K]$, we construct an estimator $\widetilde{p}_k$ of the regression function $p_k$. This estimation step is standard and has been extensively studied in the literature. In practice, any suitable machine learning method can be employed, such as kernel-based estimators or random forests.
To derive theoretical guarantees on the excess risk, we require that the estimated scores satisfy a continuity property analogous to Assumption~\ref{assum:continuous}. To enforce this, we introduce a (small) perturbation: let $\epsilon \sim \mathcal{U}([0,10^{-\eta}])$ be an independent random noise, independent from all other data. This additive noise ensures that ties in the estimated probabilities occur with probability zero, without impacting the validity of the procedure. We thus define the final estimator of the class-conditional probabilities as $\widehat{p}_k(x,s) = \widetilde{p}_k(x,s) + \epsilon$.

In a second step, based on the unlabeled dataset $\mathcal{D}_N$, we estimate both the distribution of the sensitive attribute $S$ and the Lagrangian parameters $(\lambda^*, \gamma^*)$.
For each $s \in \mathcal{S}$, we build the subset $\mathcal{D}_{N_s} = \left\{(X_i, S_i) \in \mathcal{D}_N : S_i = s\right\}$, with corresponding size $N_s = \sum_{i=n+1}^{n+N} \1{S_i = s}$.
The distribution $(\pi_s)_{s \in \mathcal{S}}$ is estimated using the empirical frequencies $(\widehat{\pi}_s)_{s \in \mathcal{S}}$ with $\widehat{\pi}_s = \frac{N_s}{N}$.
Next, inspired by the Lagrangian formulation in Equation~\eqref{eq:set-valued fair lagrangian2}, we define the empirical parameters $(\widehat{\lambda}, \widehat{\gamma})$ as the solution of the following convex optimization problem:
\begin{equation}
\label{eq:EmpiricalLagrangian}
(\widehat{\lambda},  \widehat{\gamma}) \in \argmin{\substack{(\lambda, \gamma) \in \mathbb{R}^{K|\mathcal{S}|+1} \\ \lambda \geq 0 \\ \sum_{s \in \mathcal{S}}\gamma_{k,s} = 0 }} \sum_{k = 1}^K\sum_{s \in \mathcal{S}} \dfrac{1}{N_s}\sum_{i \in \mathcal{D}_{N_s}}{\left( \widehat{\pi}_s \left(\widehat{p}_k(X_i, s) - \lambda \right) - \gamma_{k,s}\right)_{+}} +\lambda\beta \enspace.
\end{equation}
The final predictor $\widehat{\Gamma}_\beta$ is then defined pointwise, using these estimated parameters (the complete estimation procedure is summarized in Algorithm~\ref{algo:procedure}):
\begin{equation*}
    \widehat{\Gamma}_{\beta}(x,s) = \left\{ k \in [K] : \widehat{p}_k(x,s) \geq \widehat{\lambda} + \frac{\widehat{\gamma}_{k,s}}{\widehat{\pi}_s} \right\} \enspace.
\end{equation*}
\begin{algorithm}[H]\label{algo:procedure}
\caption{Fair Set-Valued Classification Procedure}
\KwData{Unlabeled dataset $\mathcal{D}_N = \left\{(X_i, S_i)\right\}_{i=1}^N$, number of classes $K$, sensitive set $\mathcal{S}$, estimators $(\widetilde{p}_k)_{k=1}^K$, perturbation level $\eta$}
\KwResult{Fair set-valued classifier $\widehat{\Gamma}_\beta$}

\vspace{0.5em}
\textbf{Step 1: Add random noise to avoid ties}
\For{$k \in [K]$}{
    $\widehat{p}_k(x,s) \gets \widetilde{p}_k(x,s) + \epsilon$, with $\epsilon \sim \mathcal{U}(0, 10^{-\eta})$
}

\vspace{0.5em}
\textbf{Step 2: Estimate sensitive attribute distribution}
\For{$s \in \mathcal{S}$}{
    $N_s \gets \sum_{i=1}^N \1{S_i = s}$ \\
    $\widehat{\pi}_s \gets \frac{N_s}{N}$
}

\vspace{0.5em}
\textbf{Step 3: Solve the empirical Lagrangian problem} \\
$(\widehat{\lambda}, \widehat{\gamma}) \gets$ solution of Equation~\eqref{eq:EmpiricalLagrangian} using $(\widehat{p}_k)$ and $(\widehat{\pi}_s)$

\vspace{0.5em}
\textbf{Step 4: Define the empirical classifier} \\
\ForEach{$(x,s) \in \mathcal{X} \times \mathcal{S}$}{
    $\widehat{\Gamma}_{\beta}(x,s) \gets \left\{ k \in [K] : \widehat{p}_k(x,s) \geq \widehat{\lambda} + \frac{\widehat{\gamma}_{k,s}}{\widehat{\pi}_s} \right\}$
}
\end{algorithm}


\subsection{Rates of Convergence}
\label{subsec:RatesofConv}
The previous section introduced the plug-in set-valued predictor. We now turn to its theoretical performance, focusing on finite-sample guarantees in terms of expected size and fairness constraint violation. To this end, we quantify fairness violation through the following unfairness measure:
\begin{equation*}
\mathcal{U}(\Gamma) = \underset{k, s, s'}{\max}\left\{\left\lvert\proba[X|S=s]{k \in \Gamma(X,s)} - \proba[X|S=s']{k \in \Gamma(X,s')}\right\rvert\right\}\enspace.
\end{equation*}
This definition extends the fairness measure introduced in~\cite{Denis_Elie_Hebiri_Hu_2024} for single-output multiclass classifiers to the set-valued prediction setting. It captures the largest discrepancy, across all class labels and sensitive groups, in the probability that a class is selected by the predictor.
We can now state our first result, which shows that the plug-in predictor derived in~\cref{algo:procedure} satisfies both fairness and expected size constraints at a controlled rate:
\begin{theorem}[Fairness and Expected size controls]
\label{thm:Unfairness bound}
    Let $\widehat{\Gamma}_{\beta}$ be the empirical DP-fair set-valued classifier resulting from~\cref{algo:procedure}.
    Then, for any data-generating distribution $\proba[]{}$ and any estimators $\widetilde{p}_k$ of the class-conditional probabilities, the following bounds hold:
    \begin{align*}
        \E{\mathcal{U}(\widehat{\Gamma}_{\beta})} &\leq \frac{CK}{\sqrt{N}} \enspace ,\\
        \E{\left\lvert \mathcal{T}\left(\widehat{\Gamma}_{\beta}\right)- \beta \right\rvert} &\leq \frac{CK}{\sqrt{N}}\enspace,
    \end{align*}
    where $C>0$ is a universal constant.
\end{theorem}
The above result provides distribution-free guarantees: it holds uniformly over all distributions $\proba[]{}$ and all base estimators $(\widetilde{p}_k)_{k=1}^K$. It shows that the proposed method closely mimics the oracle $\beta$-specific DP-fair set-valued classifier $\Gamma_{\beta}^*$ -- which satisfies both constraints exactly — at a parametric rate.
These bounds combine and extend the results of~\cite{Denis_Hebiri_2017} and~\cite{Denis_Elie_Hebiri_Hu_2024}, by simultaneously addressing both the fairness constraint and the expected size constraint in the more general set-valued classification setting. 
Overall, for both bound, we get a linear cost in $K$ for handling multiple outputs and a convergence rate proportional to $1 / \sqrt{N}$ with respect to the size of the unlabeled dataset.
\\
We now turn to bounding the risk of the empirical set-valued classifier and compare it to the optimal fair predictor that satisfies both the fairness and size constraints.
\begin{theorem}[excess-risk control]\label{thm:excess risk}
Let $\widehat{\Gamma}_{\beta}$ be the empirical DP-fair set-valued classifier resulting from~\cref{algo:procedure}. Let $\Rcal_{\lambda^*, \alpha^*}(\cdot)$ the set-valued risk from~Proposition~\ref{prop:charctOptimRisk}. Then we have 
    \begin{align*}
        &\Rcal_{\lambda^*, \alpha^*}(\widehat{\Gamma}_{\beta}) - \Rcal_{\lambda^*, \alpha^*}(\Gamma^*_{\beta}) \leq  C_{K,\Scal} \left( \frac{1}{\sqrt{N} } + \underset{s\in\Scal}{\max}\norm[\infty, \mathbb{P}_{X|S=s}]{\widehat{p}- p}\right) \enspace ,
    \end{align*}
where $C_{K,\Scal}>0$ depends only on $K$ and $|\Scal|$ and $\norm[\infty, \mathbb{P}_{X|S=s}]{\widehat{p}- p} = \mathbb{E}_{\mathcal{D}_n} { \sup_{x\in\mathcal{X}} \left\lvert \widehat{p}_k(x,s) - p_k(x,s) \right\rvert }$ for all $s \in \Scal$ with $\mathbb{E}_{\mathcal{D}_n}$ being the expectation \emph{w.r.t.} the law of $\mathcal{D}_n$.
\end{theorem}
This bound is composed of two terms: the first, of order $1/\sqrt{N}$, reflects the impact of estimating the Lagrange multipliers based solely on unlabeled data and governs the control of constraint violations; the second term corresponds to the estimation error of the conditional class probabilities $(p_k)_{k \in [K]}$.
In particular, the result shows that the plug-in estimator $\widehat{\Gamma}_{\beta}$ performs nearly as well as the oracle predictor $\Gamma^*_{\beta}$, provided the class probability estimators converge uniformly, {\it i.e.,} $\underset{s\in\Scal}{\max}\norm[\infty, \mathbb{P}_{X|S=s}]{\widehat{p}- p}$ tends to $0$ as the number of labeled data $n$ tends to $\infty$.
Moreover, under additional regularity assumptions on the regression functions ({\it e.g.,} Lipschitz continuity), one can derive explicit convergence rates depending on the choice of the estimators $\widehat{p}_k$. Such rates are well studied in the literature for various methods such as $k$-nearest neighbors, kernel estimators, or random forests (see, e.g.,~\citep{GyorfiBook2002}). 
Finally, faster convergence rates can be obtained by leveraging margin-type assumptions~\cite{audibert_tsybakov_2007}. Such conditions are known to sharpen excess-risk bounds in classification tasks and are also considered in~\cite{Denis_Elie_Hebiri_Hu_2024}.

%

\section{A two-step procedure: size-to-fairness set-valued classifier}
\label{sec:twoStepProc}
The previous section focused on a plug-in approach that approximates the $\beta$-specific DP-fair oracle~$\Gamma^*_\beta$. While the resulting set-valued predictor $\widehat{\Gamma}_\beta$ is nearly optimal in terms of risk, fairness, and size constraint satisfaction, it involves solving the optimization problem~\eqref{eq:set-valued fair lagrangian2}, which may be computationally expensive. Although smoothing techniques (see, e.g.,~\cite{nesterov2012make}) can accelerate this step, it remains of interest to design simpler, more efficient alternatives.
In this section, we introduce an alternative approach, termed the \emph{size-to-fairness set-valued classifier}, which yields promising empirical performance. The core idea is to start from a potentially \emph{unfair} set-valued classifier that satisfies the size constraint and subsequently correct it to enforce fairness. This two-step procedure is described in Section~\ref{subsec:two_step_description}, and its theoretical properties are discussed in Section~\ref{subsec:two_step_properties}.
\subsection{Description of the procedure}
\label{subsec:two_step_description}
The method builds upon the characterization given in Proposition~\ref{prop:uniqueness}. Let $\widetilde{\Gamma}$ denote a set-valued classifier with expected size $\beta$, and define, for each $k \in [K]$, the marginal inclusion rate $\beta_k = \mathbb{P}\left( k \in \widetilde{\Gamma}(X,S) \right)$.
We define the associated \emph{fair} set-valued classifier $\widetilde{\Gamma}_{\mathrm{fair}}$ by thresholding each conditional probability using the quantiles of the stratum-specific distributions:
\begin{equation*}
\widetilde{\Gamma}_{{\rm fair}}(x,s) = \left\{k \in [K], \;\; p_k(x,s) \geq \bar{F}^{-1}_{k,s}(\beta_k) \right\} \enspace .
\end{equation*}
Under Assumption~\ref{assum:continuous}, for every $k \in [K]$ and $s \in \mathcal{S}$, we have
\begin{equation*}
\mathbb{P}_{X|S=s}\left(k \in  \widetilde{\Gamma}_{{\rm fair}}(X,S) \right) = \bar{F}_{k,s}\left(\bar{F}_{k,s}^{-1}(\beta_k)\right) = \beta_k \enspace,
\end{equation*}
which ensures that this new predictor satisfies the Demographic Parity (DP) constraint. Moreover, the expected size of $\widetilde{\Gamma}_{\mathrm{fair}}$ remains $\beta$, since the transformation preserves marginal inclusion rates. Both fairness and size constraints are thus simultaneously met.

To specify $\widetilde{\Gamma}$, we follow the construction proposed in~\cite{Denis_Hebiri_2017} and define
\begin{equation*}
\widetilde{\Gamma}(X,S) = \left\{k \in [K], \;\; p_k(X,S) \geq G^{-1}(\beta)\right\} \enspace, 
\end{equation*}
where $G^{-1}(\beta)$ is the global threshold ensuring $\mathcal{T}(\widetilde{\Gamma}) = \beta$. This classifier is known to solve
\begin{equation*}
\widetilde{\Gamma} \in  \arg\min \{R(\Gamma) \;\; {\it s.t.} \;\; \mathcal{T}(\Gamma) \leq \beta\} \enspace ,  
\end{equation*}
and is thus optimal among all size-constrained predictors, albeit not necessarily fair. Furthermore, we have $\mathbb{P}\left(k \in \widetilde{\Gamma}(X,S)\right) = \bar{F}_k\left(G^{-1}(\beta)\right)$. 
Combining the above steps, we obtain the \emph{size-to-fair} set-valued classifier $\widetilde{\Gamma}_{\beta 2 \mathrm{DP}}$, defined as:
\begin{equation*}
\widetilde{\Gamma}_{\beta 2 DP}(x,s) = \left\{k \in [K], \;\; p_k(x,s) \geq \bar{F}_{k,s}^{-1}\left(\bar{F}_k\left(G^{-1}(\beta)\right)\right)\right\} \enspace . 
\end{equation*}
We emphasize that although the predictor $\widetilde{\Gamma}_{\beta 2 DP}$ achieves both the expected size and demographic parity constraints by construction, it may not be optimal in terms of risk minimization. This stems from the fact that the thresholds defining $\widetilde{\Gamma}_{\beta 2 DP}$ might differ from those of the optimal DP-fair predictor $\Gamma_\beta^*$, which explicitly minimizes the risk under fairness and size constraints.

The following remark illustrates this discrepancy under additional assumptions ensuring the uniqueness of the Lagrange multipliers and the oracle predictor. It makes explicit the gap between $\widetilde{\Gamma}_{\beta 2 DP}$ and $\Gamma_\beta^*$, and justifies the suboptimality (in risk) of the two-step procedure.
\begin{remark}
\label{rk:uniqueness}
Assume that for all $k \in [K]$, the random variable $p_k(X,S)$ admits a strictly positive continuous density -- that is Assumption~\ref{assum:positive density} -- and recall that in this case the oracle thresholds $\bar{F}^{-1}_{k,s}(\beta_k^*)$ with $\beta_k^* = \mathbb{P}(k \in \Gamma_\beta^*(X,S))$ from Proposition~\ref{prop:uniqueness} are uniquely defined. Due to the non-linearity of the quantile operator, the composition of quantiles does not commute, so in general, we have 
\begin{equation*}
\beta_k^* \neq \bar{F}_k(G^{-1}(\beta)) \enspace.
\end{equation*}
As a consequence, the thresholds used to define the fair correction in the two-step predictor $\widetilde{\Gamma}_{\beta 2 DP}$ differ from those of the oracle predictor $\Gamma_\beta^*$. That is, $\bar{F}^{-1}_{k,s}(\beta_k^*) \neq \bar{F}^{-1}_{k,s}\left( \bar{F}_k(G^{-1}(\beta)) \right)$, which implies:
\begin{equation*}
\widetilde{\Gamma}_{\beta 2 DP} \neq \Gamma^*_\beta \qquad \text{and} \qquad R(\Gamma^*_\beta) < R(\widetilde{\Gamma}_{\beta 2 DP}) \enspace.
\end{equation*}
\end{remark}
Despite this potential gap in risk, the size-to-fairness predictor $\widetilde{\Gamma}_{\beta 2 DP}$ retains strong advantages: it is easily implementable, requires no constrained optimization, and offers robust constraint satisfaction. As we will highlight in Section~\ref{sec:Experiments and Results}, it also exhibits competitive numerical performance in practice.
In the next paragraph, we introduce a data-driven implementation of this procedure based on the plug-in principle.

\paragraph*{Two-step plug-in predictor.}
We now construct a plug-in estimator $\widehat{\Gamma}_{\beta 2 DP}$ of the size-to-fairness predictor $\widetilde{\Gamma}_{\beta 2 DP}$, based on the labeled and unlabeled data.
To this end, we consider the same estimators $\widehat{p}_k$ built from the labeled dataset $\mathcal{D}_n$ as in Section~\ref{subsec:procedure}. Then, using the unlabeled dataset $\mathcal{D}_N$, we define the empirical cumulative distribution functions for each $k \in [K]$ and $s \in \mathcal{S}$ as:
\begin{equation*}
\widehat{F}_{k, s}(\cdot) = \frac{1}{N_s}\underset{i = 1}{\overset{N}{\sum}} \1{S_i = s} \1{\widehat{p}_k(X_i,S_i) \leq \cdot}\enspace, \qquad\qquad
\widehat{F}_{k}(\cdot) = \sumover{s \in \Scal} \widehat{\pi}_s \widehat{F}_{k,s}(\cdot)\enspace,
\end{equation*}
where $N_s = \sum_{i=1}^N \1{S_i = s}$ and $\widehat{\pi}_s = N_s / N$.  
We denote the associated empirical survival functions by $\widehat{\overline{F}}_{k,s} = 1 - \widehat{F}_{k,s}$ and $\widehat{\overline{F}}_{k} = 1 - \widehat{F}_{k}$.  
We also define the empirical version of the function $G$ as
\begin{equation*}
\widehat{G} (\cdot) = \sum_{k=1}^{K} \widehat{\overline{F}}_{k}(\cdot) \enspace.
\end{equation*}
With this notation, we define the size-to-fairness plug-in predictor as:
\begin{equation}
\label{eq:eqTwoStepEmp}
\widehat{\Gamma}_{\beta 2 DP}(x,s) = \left\{k \in [K] : \widehat{p}_{k}(x,s) \geq \widehat{\overline{F}}_{k,s}^{-1}\left(\widehat{\overline{F}}_k(\widehat{G}^{-1}(\beta))\right) \right\} \enspace.
\end{equation}

\subsection{Statistical properties}
\label{subsec:two_step_properties}
The previous discussion highlights that, since the population-level predictor $\widetilde{\Gamma}_{\beta 2 DP}$ is not necessarily risk-optimal, our main focus lies in assessing whether its plug-in estimator $\widehat{\Gamma}_{\beta 2 DP}$ satisfies the desired size and fairness constraints.
\paragraph*{Constraint guarantees.}
We first establish that $\widehat{\Gamma}_{\beta 2 DP}$ achieves control over the expected size and demographic parity constraints, up to a deviation of order $1/\sqrt{N}$. Notably, these results are non-asymptotic and distribution-free.
\begin{theorem}\label{thm:two-stepMethodBounds}
Let $\widehat{\Gamma}_{\beta 2 DP}$ be the empirical predictor defined by Equation~\eqref{eq:eqTwoStepEmp}. There exists a constant $C > 0$ such that
\begin{align*}
    \E{\mathcal{U}(\widehat{\Gamma}_{\beta 2 DP})} &\leq \frac{CK}{\sqrt{N}} \enspace ,\\
    \E{\left\lvert \mathcal{T}\left(\widehat{\Gamma}_{\beta 2 DP}\right) - \beta \right\rvert} &\leq \frac{CK}{\sqrt{N}}\enspace.
\end{align*}
\end{theorem}
These convergence rates are of the same order as those obtained for the optimal plug-in predictor $\widehat{\Gamma}_\beta$ studied in Section~\ref{subsec:RatesofConv}. In particular, the $1/\sqrt{N}$ rate reflects the statistical error in estimating the cumulative distribution functions $F_{k,s}$ from the unlabeled dataset. Importantly, this means the method benefits from unlabeled data alone, which is advantageous in practice.

\paragraph*{Computational considerations.}
While $\widehat{\Gamma}_{\beta 2 DP}$ does not benefit from the optimality guarantees in terms of risk, it offers strong computational advantages over the estimator $\widehat{\Gamma}_\beta$. Both procedures rely on preliminary estimators $(\widehat{p}_k(X_i,S_i))_{i=1}^N$, so we do not include their cost in the complexity analysis.
Let $M$ denote the cost of one step of the numerical optimizer and $T$ the number of iterations needed for convergence. Then: i) the plug-in estimator $\widehat{\Gamma}_\beta$ has overall time complexity of order $\mathcal{O}(MTK|\Scal|N)$; ii) in contrast, the two-step method $\widehat{\Gamma}_{\beta 2 DP}$ requires only $\mathcal{O}(K\size{S}N)$ operations, as it reduces to empirical quantile computations.

Furthermore, the constant $M$ can in practice grow with $K$, $|\Scal|$, $N$, and even the target size level $\beta$ (observed empirically for the latter). For instance, using the \texttt{BFGS} optimization algorithm, one may encounter complexities of order at least $\mathcal{O}(T(K\size{\Scal})^3N)$.
Empirically, we observe that the two-step method scales significantly better with the number of classes. On a toy example, increasing $K$ results in a clear speed-up for $\widehat{\Gamma}_{\beta 2 DP}$ compared to $\widehat{\Gamma}_\beta$ (see~\cref{fig:two-step vs optimizer time perfs} in Section~\ref{sec:Experiments and Results}). This makes the two-step predictor a promising alternative for large-scale applications.

\section{Experiments and Numerical Results}
\label{sec:Experiments and Results}

This section presents the empirical evaluation of the algorithms developed in Sections~\ref{sec:Data Driven Procedure} and~\ref{sec:twoStepProc}, using both synthetic and real-world datasets. As a baseline, we consider the standard set-valued classifier that minimizes the risk under only a size constraint, typically resulting in unfair predictions. We denote this method by $\widehat{\Gamma}_\text{unfair}$ (labeled as \texttt{SVC} in the plots). We begin by outlining the general setup used throughout our experiments.

\subsection{Implementation}

There are two main steps involved in constructing the fair classifiers $\widehat{\Gamma}_{\beta}$ (referred to as \texttt{DP-fair SVC} in the figures) and $\widehat{\Gamma}_{\beta 2 DP}$ (referred to as \texttt{Two-Step method}): 
(1) estimating the class-conditional probabilities, and 
(2) deriving the final classifier using either optimization (solving~\cref{eq:EmpiricalLagrangian}) or plug-in estimates of quantiles and CDFs.

All experiments are implemented in Python. We use three datasets for each run: a training set composed of $n$ labeled samples and $N$ unlabeled samples, and a test set of $T$ labeled data used solely for evaluation. In the case of real-world data -- which only includes labeled observations -- we split the data into $80\%$ train / $20\%$ test, and then split the training portion again (50\%$ / 50\%$) to produce $D_n$ and $D_N$, where labels are dropped from $D_N$ to simulate unlabeled data. The same proportions are applied to synthetic datasets.

\paragraph*{Conditional probability estimation.} 
When the probabilities $p_k$ are not available ({\it i.e.,} except in the idealized synthetic setting), we estimate them using a gradient boosting algorithm 
(\texttt{GradientBoostingClassifier} from \texttt{sklearn.ensemble} with $20$ estimators). This estimation step uses only the labeled dataset $D_n$.

\paragraph*{Second step.} 
The final classifier is built using only the unlabeled dataset $D_N$. For $\widehat{\Gamma}_{\beta}$, we solve~\cref{eq:EmpiricalLagrangian} using the \texttt{L-BFGS-B} optimizer via \texttt{scipy.optimize.minimize}. For the two-step method $\widehat{\Gamma}_{\beta 2 DP}$, we compute the empirical CDFs and quantiles $\widehat{\Bar{F}}_k$ and $\widehat{\Bar{F}}_k^{-1}$ using \texttt{numpy}’s built-in functions.

\subsection{Synthetic Data}
\label{ssec:Data Generation}

We consider the case of $K=4$ classes and binary sensitive attributes $s\in \{-1,1\}$. We generate 10,000 samples from a Gaussian mixture model as follows:
\begin{align*}
Y & \sim \mathcal{M}(10{,}000 \ , (1/4,1/4,1/4,1/4)) \enspace,\\
S| &Y=k \sim \mathcal{R}ademacher\left(\frac{1}{2}+K\frac{2(k\%2)-1}{2(k+K)}\right) \enspace,\\
X| & S=s, Y=k \sim \mathcal{N}(ks\mu, I_d)\enspace, 
\end{align*}
where $I_d$ is the identity matrix in $\mathbb{R}^d$ and $\mu \sim \mathcal{U}([0,1])$.

We explore two scenarios:
\begin{enumerate}
    \item We assume access to the true conditional probabilities $p_k$ --- all corresponding plots are deferred to Appendix~\ref{append:numeric}.
    \item Probabilities $p_k$ are unknown and must be estimated from data.
\end{enumerate}

\paragraph*{Comparison with the Unfair Baseline.}  
Figures~\ref{fig:results on generated data no estimation} and~\ref{fig:results on generated data} respectively show the results for the two scenarios above. Subplot (b) in each figure confirms that both fair and unfair classifiers satisfy the size constraint. This aligns with our theoretical guarantees in~\cref{thm:Unfairness bound}, which match those in~\citep{Denis_Hebiri_2017}.
\\
Subplot (c) highlights the rise in unfairness when $\beta$ increases under the unfair classifier -- a phenomenon caused by the intrinsic bias amplification of set-valued outputs. In contrast, $\widehat{\Gamma}_\beta$ consistently yields low unfairness, close to zero.
\\
Subplot (a) shows that enforcing fairness comes with a slight increase in classification risk. However, given the large fairness improvements (reductions of up to 0.8 in unfairness), this trade-off is acceptable and expected.
\\
Finally, comparing both figures, we observe that estimating $p_k$ leads to minimal performance degradation, validating our theoretical findings in~\cref{thm:excess risk}.

\begin{figure}
\centering
\subfloat[Risk]{\includegraphics[width=0.3\textwidth]{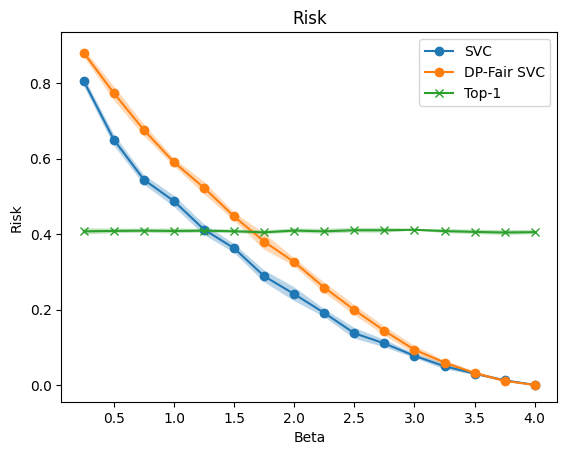}}\quad
\subfloat[Mean Size Error]{\includegraphics[width=0.3\textwidth]{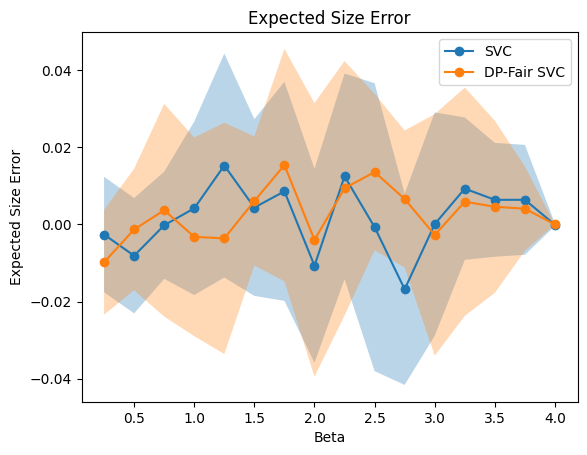}}\quad
\subfloat[Unfairness]{\includegraphics[width=0.3\textwidth]{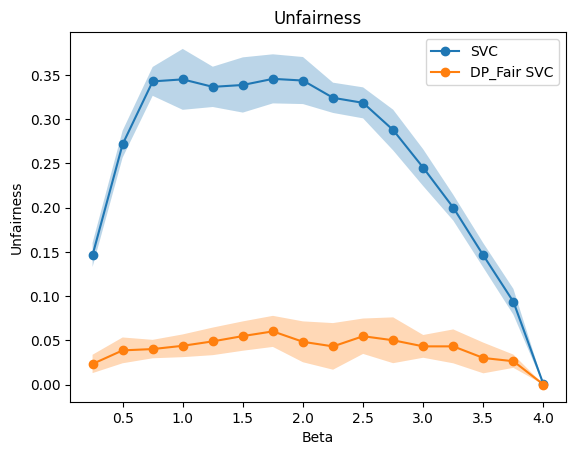}}
\caption{Results on synthetic data with estimated class-conditional probabilities (20 estimators).}
\label{fig:results on generated data}
\end{figure}

\paragraph*{Two-Step vs. Optimizer.}
We now compare the two fair classifiers $\widehat{\Gamma}_\beta$ (via optimizer) and $\widehat{\Gamma}_{\beta 2 DP}$ (two-step), both in terms of runtime and numerical stability.
\\
\textit{Runtime.}  
Figures~\ref{fig:two-step vs optimizer generated data no estimation} and~\ref{fig:two-step vs optimizer generated data} show that the two methods yield comparable performance across risk, size, and unfairness. Although $\widehat{\Gamma}_{\beta 2 DP}$ can exhibit slightly higher unfairness, this remains within acceptable bounds. The runtime advantage of the two-step method becomes clear in~\cref{fig:two-step vs optimizer time perfs}: for large values of $K$, it significantly outperforms the optimizer.
\\
\textit{Numerical Stability.}  
Figure~\ref{fig:two-step vs optimizer stability} shows that the optimizer struggles with numerical instability as the misclassification risk approaches zero, occasionally violating the constraints. In contrast, the two-step procedure remains more robust, albeit some degradation for high $\beta$.

\begin{figure}
    \centering
    \subfloat[Risk]{\includegraphics[width=0.3\textwidth]{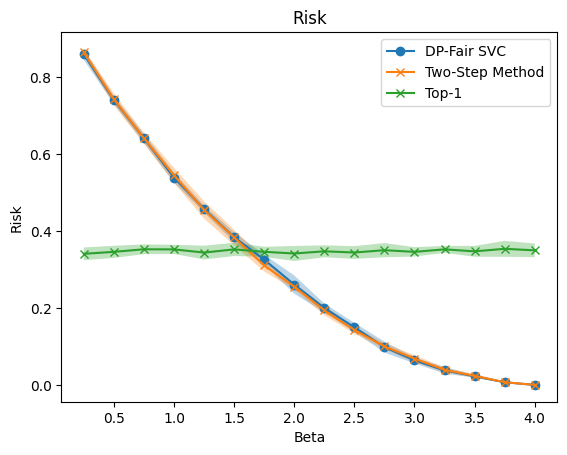}}\quad
    \subfloat[Mean Size Error]{\includegraphics[width=0.3\textwidth]{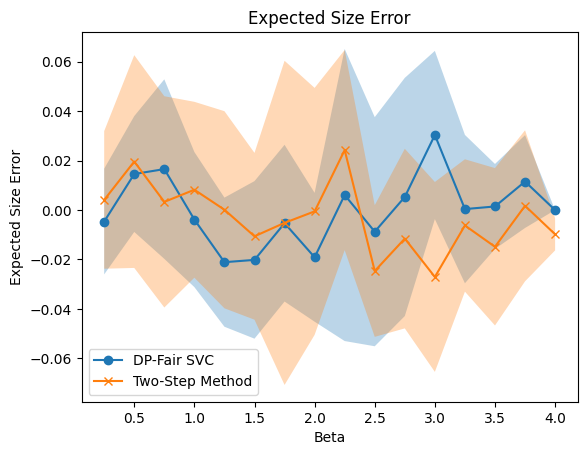}}\quad
    \subfloat[Unfairness]{\includegraphics[width=0.3\textwidth]{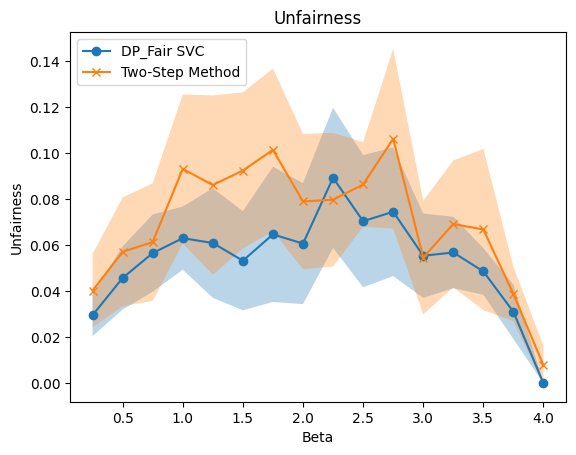}}
    \caption{Same comparison with estimated probabilities (20 estimators).}
    \label{fig:two-step vs optimizer generated data}
\end{figure}

\begin{figure}
    \centering
    \subfloat[$K=20$]{\includegraphics[width=0.3\textwidth]{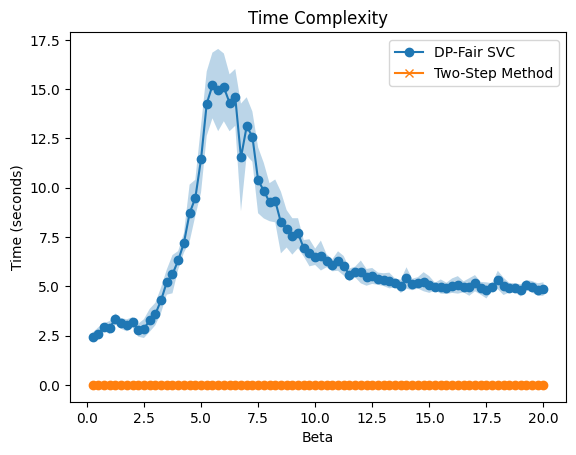}}\quad
    \subfloat[$K=50$]{\includegraphics[width=0.3\textwidth]{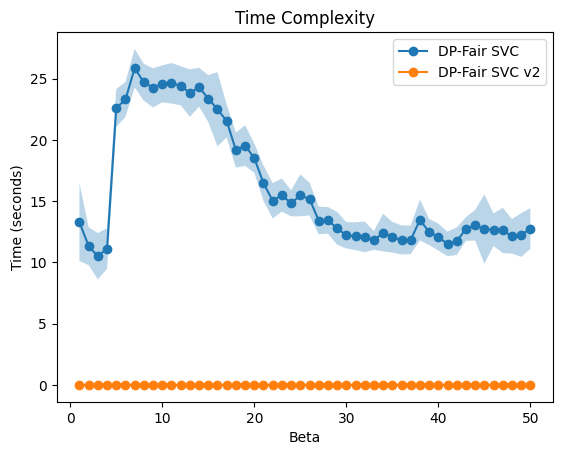}}
    \caption{Runtime comparison for increasing $K$.}
    \label{fig:two-step vs optimizer time perfs}
\end{figure}

\begin{figure}
    \centering
    \subfloat[Risk]{\includegraphics[width=0.3\textwidth]{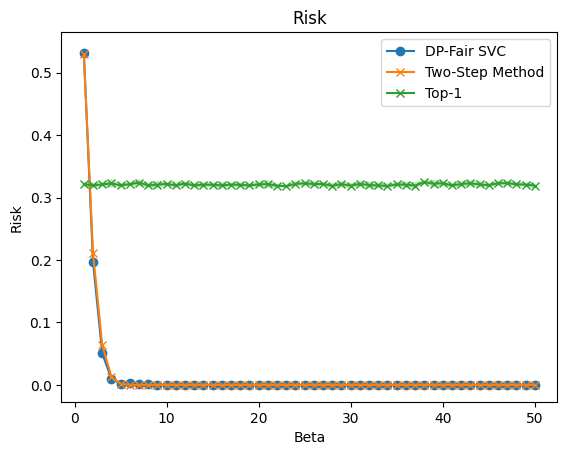}}\quad
    \subfloat[Mean Size Error]{\includegraphics[width=0.3\textwidth]{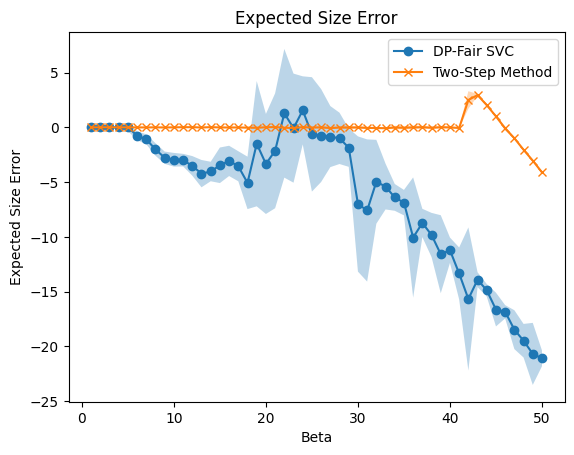}}\quad
    \subfloat[Unfairness]{\includegraphics[width=0.3\textwidth]{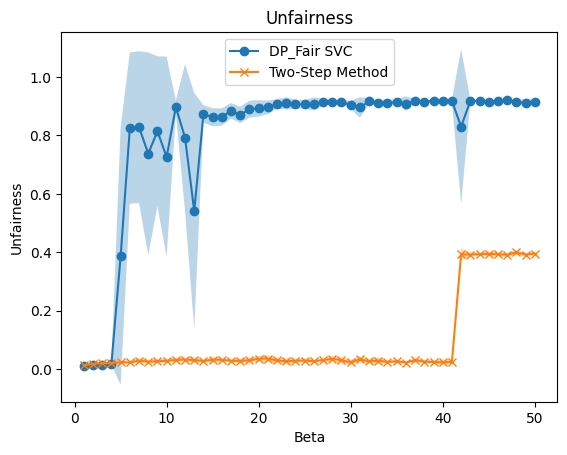}}
    \caption{Stability comparison between the optimizer and the two-step method.}
    \label{fig:two-step vs optimizer stability}
\end{figure}

\subsection{Real Data}

We now evaluate our models on the DRUG dataset\footnote{\url{https://www.kaggle.com/datasets/obeykhadija/drug-consumptions-uci}}, which contains demographic and personality data for 1,885 individuals, along with drug use behavior. The task is to predict cannabis usage levels. Following~\citep{Denis_Elie_Hebiri_Hu_2024}, we reduce the number of classes from 7 to 4: never used, not used in the past year, used in the past year but not today, and used in the past day. The sensitive attribute is binary, indicating whether the respondent has a college degree.

As shown in~\cref{fig:results on DRUG data}, our DP-fair classifier $\widehat{\Gamma}_\beta$ matches the performance of the unfair baseline in terms of risk and size accuracy, while achieving near-zero unfairness. 

\begin{figure}
    \centering
    \subfloat[Risk]{\includegraphics[width=0.3\textwidth]{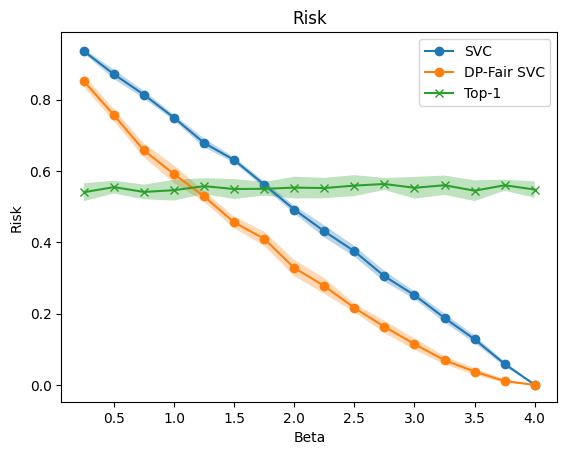}}\quad
    \subfloat[Mean Size Error]{\includegraphics[width=0.3\textwidth]{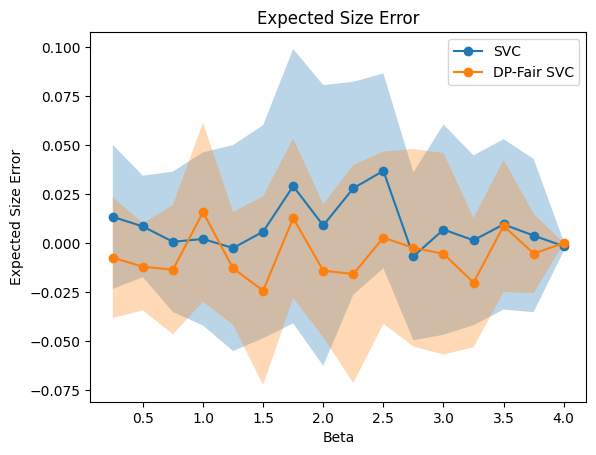}}\quad
    \subfloat[Unfairness]{\includegraphics[width=0.3\textwidth]{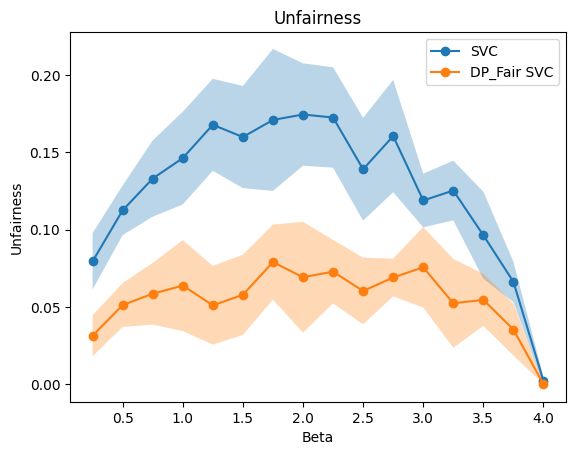}}
    \caption{Results on the DRUG dataset (20 estimators, gradient boosting).}
    \label{fig:results on DRUG data}
\end{figure}

\section{Conclusion}
\label{sec:conclusion}

In this work, we introduce a novel framework for learning \emph{fair set-valued classifiers} under Demographic Parity constraints. Unlike standard multi-class predictors, our approach outputs subsets of labels, thereby enabling a flexible control of prediction uncertainty while enforcing fairness across sensitive groups. We characterize the optimal trade-off between accuracy, fairness, and output size via an oracle construction, and propose two practical algorithms: a plug-in estimator based on constrained optimization, and a computationally efficient two-step correction procedure. Both methods rely solely on unlabeled data for enforcing fairness, making them appealing for real-world applications where labeling might be expensive or sensitive.
\\
Our framework offers an alternative to popular fairness relaxations such as $\varepsilon$-fairness, where the choice of the tolerance parameter $\varepsilon$ is often arbitrary and lacks interpretability. In contrast, the set-valued formulation enables a direct and meaningful control of the predictor’s output size, which provides both interpretability and tunability from a practitioner’s perspective. This makes our approach a compelling and principled substitute for unconstrained or approximately constrained fairness objectives.

Beyond empirical performance and constraint guarantees, the set-valued perspective opens promising research directions. In particular, future work could investigate how to extend fairness-aware prediction to structured output problems, such as hierarchical classification. Moreover, extending the fairness constraint to more general criteria (such as equalized odds or individual fairness) in the set-valued prediction setting is an open challenge.

Overall, our results suggest that fair set-valued prediction is a versatile and powerful tool for bridging the gap between predictive performance, fairness, and interpretability.

\cleardoublepage\bibliography{references}
\bibliographystyle{plainnat}

\newpage
\cleardoublepage\appendix

\begin{center}
\Large{\textbf{Appendix}}
\end{center}
\bigskip

This appendix consists mainly in two parts. A first one is devoted to additional numerical results (\cref{append:numeric}). In a second part (\cref{append:proofs}), we gather the proof of our results.

\section{Additional numerical results}
\label{append:numeric}
This section provides all plots related to the the performance of the set-valued classifiers when the class-conditional probabilities are known. As compared to the case where the class-conditional probabilities are unknown, we observe that estimating $p_k$ leads to minimal performance degradation, validating our theoretical findings in~\cref{thm:excess risk}.

\begin{figure}
\centering
\subfloat[Risk]{\includegraphics[width=0.3\textwidth]{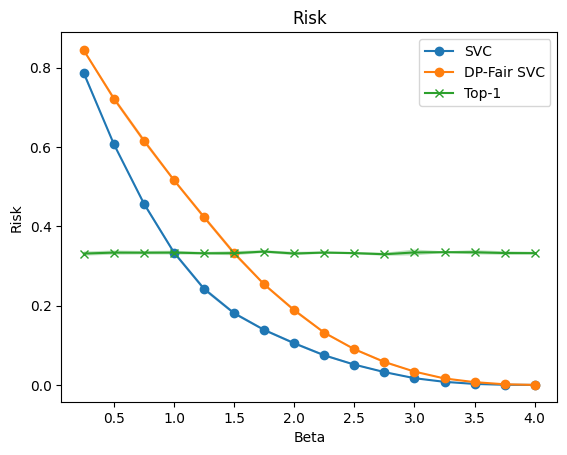}}\quad
\subfloat[Mean Size Error]{\includegraphics[width=0.3\textwidth]{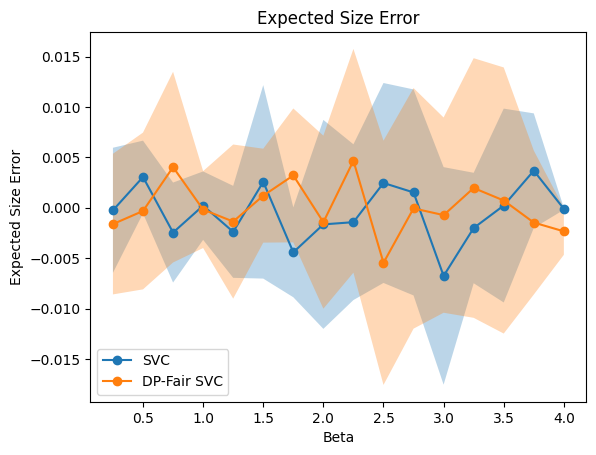}}\quad
\subfloat[Unfairness]{\includegraphics[width=0.3\textwidth]{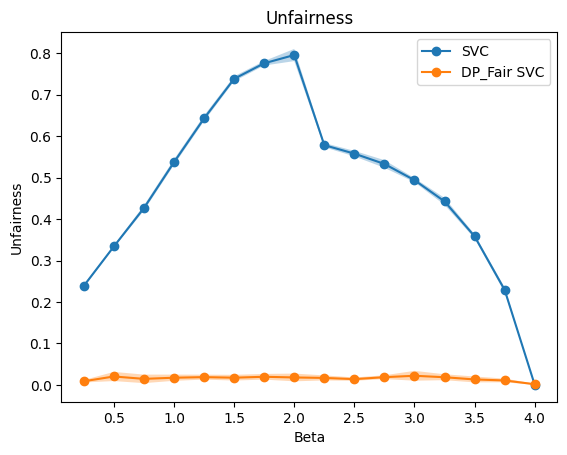}}
\caption{Results on synthetic data using the true conditional distributions.}
\label{fig:results on generated data no estimation}
\end{figure}

\begin{figure}
    \centering
    \subfloat[Risk]{\includegraphics[width=0.3\textwidth]{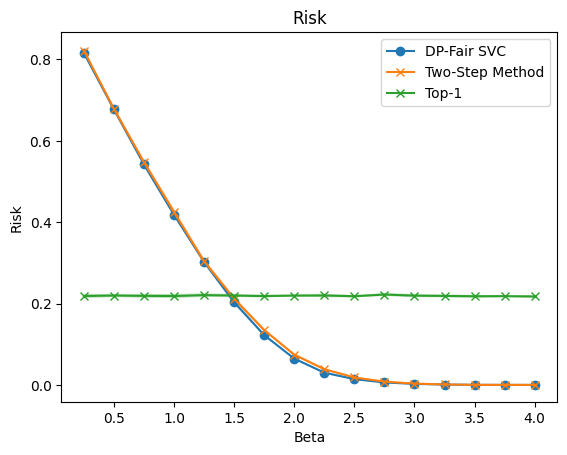}}\quad
    \subfloat[Mean Size Error]{\includegraphics[width=0.3\textwidth]{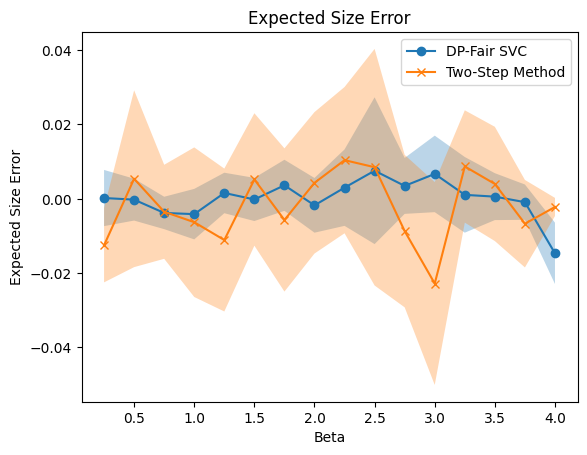}}\quad
    \subfloat[Unfairness]{\includegraphics[width=0.3\textwidth]{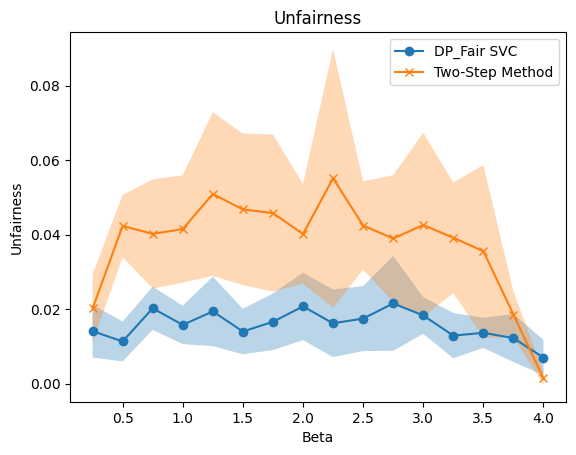}}
    \caption{Comparison between optimizer-based and two-step fair classifiers (true probabilities).}
    \label{fig:two-step vs optimizer generated data no estimation}
\end{figure}

In addition, Figure~\ref{fig:two-step vs optimizer DRUG} displays the two-step set-valued classifier performance as compared to the optimizer-based approach in the case of the real data. The conclusion are similar to the case of the synthetic data: risks, size, and unfairness are comparable for both approaches.

\begin{figure}
    \centering
    \subfloat[Risk]{\includegraphics[width=0.3\textwidth]{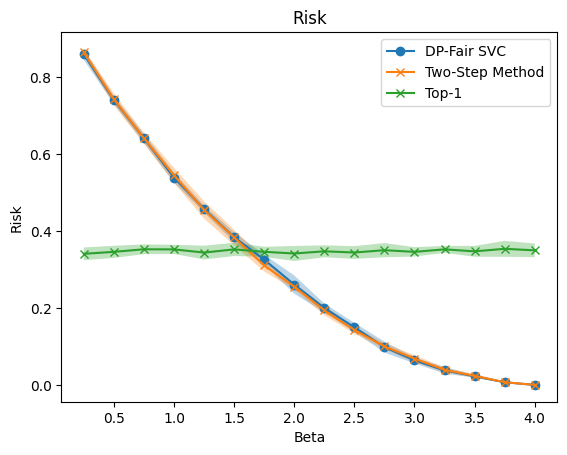}}\quad
    \subfloat[Mean Size Error]{\includegraphics[width=0.3\textwidth]{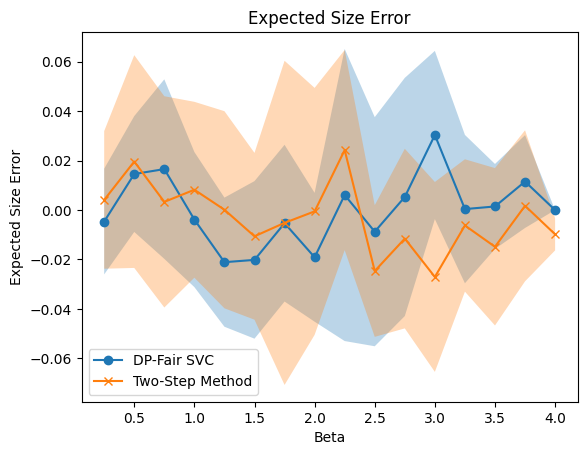}}\quad
    \subfloat[Unfairness]{\includegraphics[width=0.3\textwidth]{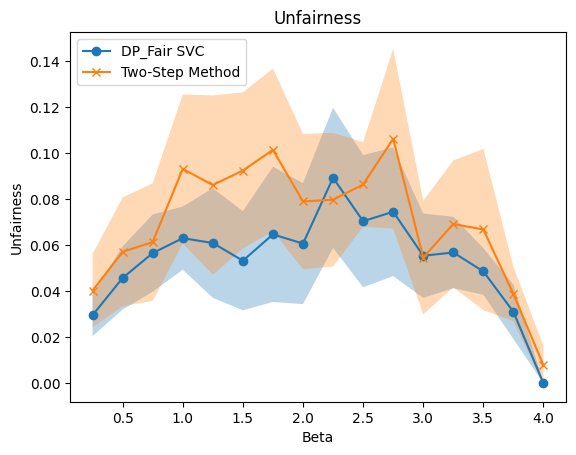}}
    \caption{Two-step vs optimizer on the DRUG dataset.}
    \label{fig:two-step vs optimizer DRUG}
\end{figure}

\section{Proofs}
\label{append:proofs}
This appendix is organised as follows: in Section~\ref{sec:tecRes}, we state important technical tools that will be used in the main proof section. The remaining of this appendix is devoted to the proofs of the results in~\cref{sec:General Framework,sec:Data Driven Procedure,sec:twoStepProc} respectively.

We also introduce $\Delta = \{\left(\lambda, \gamma\right) \in \mathbb{R}^{K|\mathcal{S}|+1}, \;\; \lambda \geq 0, \;_; \sum_{s \in \mathcal{S}}\gamma_{k,s} = 0\}$ the set of parameters of interests. 
Throughout the proofs, we use the following notation
\begin{equation*}
\widehat{\mathbb{P}}_{X|S=s}\left(\widehat{p}_{k,s} \geq \cdot \right)  = \dfrac{1}{N_s} \sum_{i \in \mathcal{D}_{N_s}} \1{\widehat{p}_{k}(X_i,s) \geq \cdot}.
\end{equation*}

\subsection{Technical Results}
\label{sec:tecRes}

    \begin{lemma}\label{lem:1_over_n_divergence}

        Let $Z_1, \ldots ,Z_N$ i.i.d random variable with continuous distribution function.
        Let us denote by $\widehat{F}$ the empirical cumulative distribution function.
        For each $u \in (0,1)$, we have 
        \begin{equation*}
        0 \leq \widehat{F}(\widehat{F}^{-1}(u)) -u \leq \dfrac{1}{N}, \;\; a.s.
        \end{equation*}
    \end{lemma}
    
    \begin{proof}
        Let $\sigma$ be the ordering permutation ensuring $Z_{\sigma(i)} < Z_{\sigma(i+1)}$ almost surely for $i \in [N-1]$.
        Assume that for some $i \in \left\{ 2, \ldots, N \right\}$, $u \in [\hat{F}(Z_{\sigma(i-1)}), \hat{F}(Z_{\sigma(i)}))$. Therefore $\hat{F}^{-1}(u) = Z_{\sigma(i)}$ and then $\hat{F}\left(\hat{F}^{-1}\left( u \right)\right) = \hat{F}(Z_{\sigma(i)}) = \frac{i}{N}$. By subtracting u, thanks to the continuity assumption, we get
        \begin{align*}
            0 = \hat{F}(Z_{\sigma(i)}) - \hat{F}(Z_{\sigma(i)}) &\leq \hat{F}\left(\hat{F}^{-1}\left( u \right) \right) - u\\ &= \hat{F}(Z_{\sigma(i)}) - u\\ &\leq \hat{F}(Z_{\sigma(i)}) - \hat{F}(Z_{\sigma(i-1)}) = \frac{1}{N} \enspace a.s.
        \end{align*}
        And for $u \in (0,\widehat{F}(Z_{\sigma(1)}))$, similar reasoning holds with $\hat{F}^{-1}(u) = Z_{\sigma(1)}$.
    \end{proof}

\begin{lemma}
\label{lem:BinomialInverse}
Let $Z$ a random variable distributed according to a Binomial distribution with parameter $N,p$.
We have that
\begin{equation*}
\mathbb{E}\left[\dfrac{\1{Z > 0}}{Z}\right] \leq \dfrac{2}{(N+1)p} \enspace .
\end{equation*}
\end{lemma}

\begin{lemma}\label{lem:G properties}
(Proposition 1 of~\cite{Denis_Hebiri_2017})
    Under ~\cref{assum:continuous}, the following properties hold:
    \begin{itemize}
        \item $\forall t \in (0,1), \beta \in (0,K) : G^{-1}(\beta) \leq t \iff \beta \geq G(t)$ \enspace,
        \item $\forall \beta \in (0,K), G(G^{-1}(\beta)) = \beta$ \enspace.
    \end{itemize}
\end{lemma}

\begin{lemma}(Dvoretzky-Kiefer-Wolfowitz Inequality)
\label{lem:DKW}
    Let $Z_1, \ldots, Z_N$ i.i.d. real-valued random variable with common distribution function $F$. We denote by $\widehat{F}_N$ the corresponding empirical cumulative distribution function. The following holds
    \begin{equation*}
    \mathbb{P}\left(\sup_{x \in \mathbb{R}} \left|\widehat{F}_N(x)-F(x)\right| \geq \varepsilon\right) \leq 2\exp(-2n\varepsilon^2), \qquad \forall \varepsilon > 0 \enspace.     
    \end{equation*}
\end{lemma}

\begin{lemma}
\label{lem:functionH}
    Let $\beta > 0$, and let us define the function $H$ that maps $\Delta$ onto $\mathbb{R}_{+}$ defined as follows
    \begin{equation*}
        H(\lambda, \gamma) =  \sum_{k = 1}^K\sum_{s \in \mathcal{S}} \E[X|S=s]{\left( \pi_s \left(p_k(X, s) - \lambda \right) - \gamma_{k,s}\right)_{+}} +\lambda\beta \enspace .
    \end{equation*}
The function $H$ is convex and coercive.
\end{lemma}

\begin{proof}
The convexity of $H$ is straightforward. Let us focus on the second point of the lemma.
Consider $(\lambda_m,\gamma_m)_{m \geq 0} \in \Delta$ with $\ell_2$ norm $\left\|(\lambda_m,\gamma_m)\right\| \rightarrow + \infty$. We observe that if $\lambda_m \rightarrow +\infty$, and since 
\begin{equation*}
H(\lambda, \gamma) \geq \lambda \beta \enspace,
\end{equation*}
we have $H(\lambda_m, \gamma_m) \rightarrow +\infty$ as $m \rightarrow +\infty$.
Otherwise, $\left\| \gamma_m \right\| \rightarrow  +\infty$. In this case since for each $k \in [K]$ we have the condition 
$\sum_{s \in \mathcal{S}} \gamma_{k,s} = 0$, we deduce that there exists
$k_0 \in [K]$, and $s_0 \in \mathcal{S}$ such that
\begin{equation*}
\gamma^m_{k_0,s_0} \rightarrow - \infty \;\; {\rm as} \;\; m \rightarrow +\infty \enspace.    
\end{equation*}
Therefore, we have, with $m \rightarrow +\infty$, that
\begin{equation*}
H(\lambda_m, \gamma_m) \geq  \E[X|S=s]{\left( \pi_s \left(p_k(X, s) - \lambda_m \right) - \gamma^m_{k_0,s_0}\right)_{+}} + \mathcal{O}(1) \rightarrow +\infty \enspace,
\end{equation*}
where $\mathcal{O}(1)$ is a negligible term (bounded by a constant). It then shows that $H$ is coercive.
\end{proof}

\subsection{Proof of Section~\ref{sec:General Framework}}

\begin{proof}[Proof of Theorem~\ref{thm:optimal set-valued classifier}]

In both cases, we apply the weak duality principle. More precisely,
we first solve the max-min problem
\begin{equation*}
\max_{\lambda, \boldsymbol{\alpha}} \min_{\Gamma}   \mathcal{L}\left(\Gamma, \lambda, \boldsymbol{\alpha} \right)\enspace, 
\end{equation*}
where $\mathcal{L}\left(\Gamma, \lambda, \boldsymbol{\alpha} \right)$ is the Lagrangian associated to the minimization problem. Then, we show that the solution of the max-min is an optimal fair set-valued classifier.

We first write the Lagrangian associated to our minimization problem.
For $\lambda, \boldsymbol{\alpha}, \Gamma$, we have that
\begin{multline*}
\mathcal{L}\left(\Gamma, \lambda, \boldsymbol{\alpha} \right) =     R(\Gamma) + \lambda \left(\E[X]{\left\lvert \Gamma(X,S) \right\rvert} - \beta\right) \\
        + \sumover[K]{k=1} \sumover{s \in \Scal} \alpha_{k,s} \left( \proba[X \lvert S = s]{k \in \Gamma(X,s)} - \proba[X,S]{k \in \Gamma(X,S)}\right)\enspace.
\end{multline*}
We observe that $\mathcal{L}$ can be expressed as follows
\begin{multline*}
\mathcal{L}\left(\Gamma, \lambda, \boldsymbol{\alpha} \right) = 
1- \sum_{k = 1}^K\sum_{s \in \mathcal{S}} \mathbb{E}_{X|S=s}\left[\1{k \in \Gamma(X, s)}\left(\pi_s p_k(X,s)\right)\right] \\ + \lambda\left(\sum_{k=1}^K\sum_{s \in \mathcal{S}} \mathbb{E}_{X|S=s}\left[\pi_s\1{k \in \Gamma(X, s)}\right]-\beta\right)
\\ + \sum_{k=1}^K\sum_{s\in \mathcal{S}} \mathbb{E}_{X|S=s}\left[\alpha_{k,s}\1{k \in \Gamma(X, s)}\right] - \mathbb{E}_{X|S=s}\left[\bar{\alpha}_k \1{k \in \Gamma(X, s)} \pi_s\right]\enspace,
\end{multline*}
$\bar{\alpha}_k = \sum_{s \in \mathcal{S}} \alpha_{k,s}$. Therefore, we get
\begin{equation*}
\mathcal{L}\left(\Gamma, \lambda, \boldsymbol{\alpha} \right) = 
1-  \sum_{k = 1}^K\sum_{s \in \mathcal{S}} \mathbb{E}_{X|S=s}\left[\1{k \in \Gamma(X, s)}\left(\pi_s\left(p_k(X,s)-\lambda+\bar{\alpha}_k\right) - \alpha_{k,s}\right)\right]-\lambda\beta \enspace .
\end{equation*}
From the above equation, it is not difficult to see that
\begin{equation}
\label{eq:equtile1}
\Gamma^*_{\lambda, \boldsymbol{\alpha}} \in \arg\min_{\Gamma}     \mathcal{L}\left(\Gamma, \lambda, \boldsymbol{\alpha} \right) \enspace,
\end{equation}
is also characterized pointwise as
\begin{equation*}
 \Gamma^*_{\lambda, \boldsymbol{\alpha}}(x,s) = \left\{k \in [K],\;\;  p_k(x,s) \geq \lambda + \dfrac{\alpha_{k,s}}{\pi_s} - \bar{\alpha}_k \right\}\enspace.  
\end{equation*}
Furthermore, injecting to the Lagrangian we have that
\begin{equation*}
 \mathcal{L}\left(\Gamma^*_{\lambda,\boldsymbol{\alpha} }, \lambda, \boldsymbol{\alpha} \right) = 1- \sum_{k = 1}^K\sum_{s \in \mathcal{S}} \mathbb{E}_{X|S=s}\left[\left(\pi_s\left(p_k(X,s)-\lambda+\bar{\alpha}_k\right) - \alpha_{k,s}\right)_{+}\right] -\lambda\beta \enspace.
\end{equation*}
Next, it remains to optimize in $(\lambda,\boldsymbol{\alpha})$. We have that 
\begin{equation}
\label{eq:eqLagrang1}
(\lambda^*, \boldsymbol{\alpha}^*) \in \arg\max_{\lambda, \boldsymbol{\alpha}}
\mathcal{L}\left(\Gamma^*_{\lambda,\boldsymbol{\alpha}}, \lambda, \boldsymbol{\alpha} \right) \enspace,
\end{equation}
is characterized as
\begin{equation*}
(\lambda^*, \boldsymbol{\alpha}^*) \in \arg\min_{\substack{(\lambda, \alpha) \in \mathbb{R}^{K|\mathcal{S}|+1} \\ \lambda \geq 0}} \widetilde{H}\left(\lambda, \alpha\right) \enspace,
\end{equation*}
with
$\widetilde{H}(\lambda, \boldsymbol{\alpha})=  \sum_{k = 1}^K\sum_{s \in \mathcal{S}} \mathbb{E}_{X|S=s}\left[\left(\pi_s\left(p_k(X,s)-\lambda+\bar{\alpha}_k\right) - \alpha_{k,s}\right)_{+}\right] +\lambda\beta.$
We observe that the above minimization problem can be reformulated as follows 
\begin{equation}
\label{eq:eqLagrang2}
(\lambda^*,\gamma^*) \in \arg\min_{\substack{(\lambda, \gamma) \in \mathbb{R}^{K|\mathcal{S}|+1} \\ \lambda \geq 0 \\ \sum_{s \in \mathcal{S}}\gamma_{k,s} = 0 }} \sum_{k = 1}^K\sum_{s \in \mathcal{S}} \E[X|S=s]{\left( \pi_s \left(p_k(X, s) - \lambda \right) - \gamma_{k,s}\right)_{+}} +\lambda\beta \enspace,
\end{equation}
with the introduced reparameterization $\gamma_{k,s} = \alpha_{k,s} - \pi_s \sum_s \alpha_{k,s}$. 
Let us then denote by $H$ the objective function defined as 
\begin{equation*}
H(\lambda, \boldsymbol{\gamma}) =   \sum_{k=1}^K\sum_{s \in \mathcal{S}} \mathbb{E}_{X|S=s}\left[\left(\pi_s\left(p_k(X,s)-\lambda\right) - \gamma_{k,s}\right)_{+}\right] +\lambda\beta \enspace.  
\end{equation*}
From Lemma~\ref{lem:functionH}, $H$ is convex and coercive.  Therefore there exists a global minimizer $(\lambda^*, \boldsymbol{\gamma}^*)$ that belongs to a compact subset of $\Delta$.  
Therefore, it implies that the function $\widetilde{H}$ admits also a global minimizer $(\lambda^*, \boldsymbol{\alpha}^*)$. Furthermore, thanks to Assumption~\ref{assum:continuous}, the function $\widetilde{H}$ is differentiable {\it w.r.t.} $(\lambda,\boldsymbol{\gamma})$. Therefore, we deduce from the first order condition that
$\boldsymbol{0} \in \partial \widetilde{H}(\lambda^*, \boldsymbol{\gamma}^*)$.
We then have that
\begin{equation*}
\partial_{\lambda}{\widetilde{H}}\left(\lambda^*, \boldsymbol{\alpha}^*\right)   = -\sum_{k=1}^K\sum_{s \in \mathcal{S}} \pi_s \mathbb{P}_{X|S=s}\left(p_k(X,s) \geq \lambda^* + \dfrac{\alpha^{*}_{k,s}-\pi_{s}\bar{\alpha}^{*}_{k}}{\pi_s} \right) +\beta = 0\enspace, 
\end{equation*}
which means that $\Gamma^*_\beta$ as expected size $\beta$. Furthermore
\begin{multline*}
 \partial_{\alpha_{k,s}} \widetilde{H}\left(\lambda^*, \boldsymbol{\alpha}^*\right)
 = \sum_{s' \in \mathcal{S}}\pi_{s'} \mathbb{P}_{X|S=s'}\left(p_k(X,s') \geq \lambda + \dfrac{\alpha^{*}_{k,s'}-\pi_{s'}\bar{\alpha}^{*}_{k}}{\pi_{s'}} \right) \\- \mathbb{P}_{X|S=s}\left(p_k(X,s) \geq \lambda^* + \dfrac{\alpha^{*}_{k,s}-\pi_s\bar{\alpha}^{*}_{k}}{\pi_s} \right)=0\enspace.
\end{multline*}
This means that the set-valued classifier $\Gamma_{\beta}^*$ defined, with the reparameterization $\gamma^*_{k,s} = \alpha^*_{k,s} - \pi_s \bar{\alpha}^{*}_{k}$ given by
\begin{equation}
\label{eq:eqOptimal}
\Gamma_{\beta}^*(x,s) = \left\{k \in [K], \;\;p_k(x,s) \geq \lambda^* + \dfrac{\gamma^*_{k,s}}{\pi_s} \right\}\enspace,
\end{equation}
satisfies the demographic parity constraint. To conclude the proof  we observe that $\Gamma^*_{\beta}$ satisfies
\begin{equation}
\label{eq:eqLagrang3}
\Gamma^*_{\beta}  \in \arg\min_{\Gamma}     \mathcal{L}\left(\Gamma, \lambda^*, \boldsymbol{\alpha}^* \right) \enspace.
\end{equation}
To this end we consider $\Gamma$ another set-valued classifier that satisfies the demographic parity constraint and $\mathcal{T}(\Gamma) \leq \beta$.
We have that 
\begin{multline*}
\mathcal{L}\left(\Gamma_{\beta}^*, \lambda^*, \boldsymbol{\alpha}^*\right) = R(\Gamma_{\beta}^*) \leq \mathcal{L}\left(\Gamma, \lambda^*, \boldsymbol{\alpha}^*\right)  \\ 
= R(\Gamma) + \underbrace{\lambda^*}_{\geq 0} \underbrace{\left(\mathcal{T}(\Gamma)-\beta\right)}_{\leq 0} +
\sumover[K]{k=1} \sumover{s \in \Scal} \alpha^*_{k,s} \underbrace{ \left( \proba[X \lvert S = s]{k \in \Gamma(X,s)} - \proba[X , S ]{k \in \Gamma(X,S)} \right) }_{=0} \\
\leq R(\Gamma) \enspace,
\end{multline*}
which yields the result.

\end{proof}

\begin{proof}[Proof of Proposition~\ref{prop:charctOptimRisk}]
The proof of the proposition can be easily deduced from the proof of Theorem~\ref{thm:optimal set-valued classifier}. 
The first two points of the proposition are already shown in the proof of Theorem~\ref{thm:optimal set-valued classifier}. Let us now proof the last point.
Let $(\lambda^*, \boldsymbol{\alpha}^*)$ the Lagrangian parameters defined as in Equation~\eqref{eq:eqLagrang1}. Hence, we have that $(\lambda^*, \boldsymbol{\gamma}^*)$ is a minimizer of Equation~\eqref{eq:eqLagrang2} with, for each $k \in [K], s \in \mathcal{S}$, $\gamma_{k,s}^* = \alpha_{k,s}^{*}- \pi_s \bar{\alpha}^{*}_{k}$.
Furthermore, we observe that for each set-valued classifier $\Gamma$
\begin{equation*}
    \mathcal{L}\left(\Gamma, \lambda^*, \boldsymbol{\alpha}^*\right) =
    \mathcal{R}_{\lambda^*, \boldsymbol{\gamma}^*}(\Gamma) \enspace,
\end{equation*}
which yields the result thanks to the characterization of $\Gamma_{\beta}^*$ (see Equation~\eqref{eq:eqLagrang3}).
\end{proof}

\begin{proof}[Proof of Proposition~\ref{coro:excessRisk}]
Let $\Gamma$ a set-valued classifier, and $(\lambda^*,\boldsymbol{ \gamma}^*)$ defined by Equation~\eqref{eq:set-valued fair lagrangian2}. We have that
\begin{equation*}
\mathcal{R}_{\lambda^*, \boldsymbol{\gamma}^*}\left(\Gamma\right)
= R(\Gamma) + \lambda^* \left(\E[X]{\left\lvert \Gamma(X,S) \right\rvert} - \beta\right) 
        + \sumover[K]{k=1} \sumover{s \in \Scal} \gamma^*_{k,s} \proba[X \lvert S = s]{k \in \Gamma(X,s)}.
\end{equation*}
Since 
\begin{equation*}
\E[X]{\left\lvert \Gamma(X,S) \right\rvert} = \sum_{k\in [K]}\sum_{s \in \mathcal{S}} \pi_s \mathbb{P}_{X|S=s}\left(k \in \Gamma(X,S)\right),    
\end{equation*}
and 
\begin{equation*}
1- R(\Gamma) = \sum_{k \in [K]}\sum_{s \in \mathcal{S}} \pi_s \mathbb{E}_{X|S=s}\left[\1{k \in \Gamma(X,S)}p_k(X,S)\right], 
\end{equation*}
we deduce that
\begin{multline*}
\mathcal{R}_{\lambda^*, \gamma^*}\left(\Gamma\right)-\mathcal{R}_{\lambda^*, \gamma^*}\left(\Gamma_{\beta}^*\right)
\\= \sum_{k \in [K]}\sum_{s \in \mathcal{S}}\mathbb{E}_{X|S=s}\left[\left(\pi_s\left(p_k(X,S)-\lambda\right) +\gamma_{k,s}^*\right)\left(\1{k \in \Gamma_{\beta}^*(X,S)}-\1{k \in \Gamma}\right)\right].
\end{multline*}
Now, by definition of $\Gamma_{\beta}^*$ (see Equation~\eqref{eq:eqOptimal}),
we observe that 
\begin{equation*}
\pi_s\left(p_k(X,S)-\lambda\right) +\gamma_{k,s}^* \geq 0 \;\; {\rm iff} \;\;  \1{k \in \Gamma_{\beta}^*(X,S)} = 1.
\end{equation*}
In view of this observation, we get the desired result.
\end{proof}

\begin{proof}[Proof of Proposition~\ref{prop:uniqueness}]

We start with the first point of the proposition.
Assume that there exists $(\widetilde{\lambda}, \widetilde{\gamma})$
such that for each $s \in \mathcal{S}$
\begin{equation*}
\Gamma_{\beta}^*(x,s) = \left\{k \in [K], \;\; p_k(x,s) \geq \widetilde{\lambda}+\dfrac{\widetilde{\gamma}_{k,s}}{\pi_s}\right\},    
\end{equation*}
with for $k \in [K]$, $\sum_{s \in \mathcal{S}}\widetilde{\gamma}_{k,s} = 0$.
Under Assumption~\ref{assum:positive density},
we have that for each $(k,s)$
\begin{equation*}
\mathbb{P}_{X|S=s}\left(p_k(X,S) \geq \lambda^*+ \dfrac{\gamma^*_{k,s}}{\pi_s} \right) = \mathbb{P}_{X|S=s}\left(p_k(X,S) \geq \widetilde{\lambda}+\dfrac{\widetilde{\gamma}_{k,s}}{\pi_s}\right)
\;\; {\rm iff} \;\; \lambda^* + \dfrac{\gamma^*_{k,s}}{\pi_s}
= \widetilde{\lambda}+\dfrac{\widetilde{\gamma}_{k,s}}{\pi_s}.
\end{equation*}
Since $\sum_{s \in \mathcal{S}} \pi_s = 1$, and $\sum_{s \in \mathcal{S}} \gamma_{k,s}^* = \sum_{s \in \mathcal{S}} \widetilde{\gamma}_{k,s} = 0$, we deduce from the above equality that
\begin{equation*}
    \lambda^* = \widetilde{\lambda}, 
\end{equation*}
and then $\widetilde{\gamma}_{k,s} = \gamma_{k,s}^* $ for each $(k,s) \in [K] \times \mathcal{S}$.

For the second point, we observe since $\Gamma_{\beta}^*$ satisfies the expected size constraint and the Demographic parity constraint, we deduce that for each $(k,s) \in [K] \times \mathcal{S}$ 
\begin{equation*}
\mathbb{P}\left(p_{k}(X,S) \geq\lambda^* + \dfrac{\gamma^*_{k,s}}{\pi_s}  \right)
= \mathbb{P}\left(k \in p_{k}(X,S)\right):=\beta_k^*.
\end{equation*}
Therefore, under Assumption~\ref{assum:positive density}, we have that
\begin{equation*}
\lambda^*+   \dfrac{\gamma^*_{k,s}}{\pi_s} = F_{k,s}^{-1}(\beta_k^*).  
\end{equation*}
\end{proof}

\section{Proof of Section~\ref{sec:Data Driven Procedure}}

We start this section with the following lemma.
\begin{lemma}
\label{lem:subgradient_computation}
Let $\widehat{H}$ the function each $(\lambda, \gamma) \in \Delta$
\begin{equation*}
  \widehat{H}\left(\lambda, \gamma\right) = \sum_{k = 1}^K\sum_{s \in \mathcal{S}} \dfrac{1}{N_s}\sum_{i \in \mathcal{D}_{N_s}}{\left( \widehat{\pi}_s \left(\widehat{p}_k(X_i, s) - \lambda \right) - \gamma_{k,s}\right)_{+}} +\lambda\beta.
\end{equation*}
The function $\widehat{H}$ is convex, coercive, and its subgradient is as follows
\begin{multline*}
h_{\lambda} \in \partial_{\lambda} \widehat{H}(\lambda,\gamma) \;\; {\rm iff} \;\; \exists \mu \in [0,1], \;\; h_{\lambda} = -\sum_{k = 1}^K\sum_{s \in \mathcal{S}} \dfrac{\widehat{\pi}_s}{N_s}\sum_{i \in \mathcal{D}_{N_s}} \1{\widehat{p}_{k}(X_i,s) > \lambda+\dfrac{\gamma_{k,s}}{\widehat{\pi}_s}} \\ -  \mu \sum_{k = 1}^K\sum_{s \in \mathcal{S}} \dfrac{\widehat{\pi}_s}{N_s}\sum_{i \in \mathcal{D}_{N_s}} \1{\widehat{p}_{k}(X_i,s) = \lambda+\dfrac{\gamma_{k,s}}{\widehat{\pi}_s}} +\beta,
\end{multline*}
and for each $k \in [K]$, and $s \in \mathcal{S}$,
\begin{multline*}
h_{\gamma_{k,s}} \in \partial_{\gamma_{k,s}} \widehat{H}(\lambda, \gamma)
\;\; {\rm iff} \;\; \exists \sigma_{k,s} \in [0,1], \;\;
h_{\gamma_{k,s}} = - \dfrac{1}{N_s}\sum_{i \in \mathcal{D}_{N_s}} \1{\widehat{p}_{k}(X_i,s) > \lambda+\dfrac{\gamma_{k,s}}{\widehat{\pi}_s}} \\ - \sigma_{k,s} \dfrac{1}{N_s}\sum_{i \in \mathcal{D}_{N_s}} \1{\widehat{p}_{k}(X_i,s) = \lambda+\dfrac{\gamma_{k,s}}{\widehat{\pi}_s}}.
\end{multline*}
\end{lemma}

\begin{lemma}
\label{lem:pigeonHole}
For each $k \in [K]$, and $s \in \mathcal{S}$, there exists $C > 0$ such that
\begin{equation*}
 \dfrac{1}{N_s}\sum_{i \in \mathcal{D}_{N_s}} \1{\widehat{p}_{k}(X_i,s) = \lambda+\dfrac{\gamma_{k,s}}{\widehat{\pi}_s}} \leq \dfrac{C}{N_s} \qquad  {\it a.s.}
\end{equation*}
\end{lemma}

\begin{proof}[Proof of Theorem~\ref{thm:Unfairness bound}]

We first start with the following decomposition
\begin{multline}
\label{eq:eqDecompSize1}
\mathcal{T}\left(\widehat{\Gamma}\right) -\beta
 = \mathbb{E}\left[\left|\widehat{\Gamma}(X,S)\right|\right] -\beta = \sum_{k \in [K]} \sumover{s \in \mathcal{S}} \pi_s \mathbb{P}_{X|S=s}\left(\widehat{p}_k(X,S) > \widehat{\lambda}+ \dfrac{\widehat{\gamma}_{k,s}}{\widehat{\pi}_s}\right) - \beta\\
 \sum_{k \in [K]} \sumover{s \in \mathcal{S}} \left(\pi_s-\widehat{\pi}_s\right) \mathbb{P}_{X|S=s}\left(\widehat{p}_k(X,S) > \widehat{\lambda}+ \dfrac{\widehat{\gamma}_{k,s}}{\widehat{\pi}_s}\right) +\sum_{k \in [K]} \sumover{s \in \mathcal{S}} \widehat{\pi}_s \mathbb{P}_{X|S=s}\left(\widehat{p}_k(X,S) > \widehat{\lambda}+ \dfrac{\widehat{\gamma}_{k,s}}{\widehat{\pi}_s}\right) -\beta
\end{multline}
The first term in the {\it r.h.s. } of the above equation can be bounded as follows
\begin{equation}
\label{eq:eqDecompSize2}
\left|\sum_{k \in [K]} \sumover{s \in \mathcal{S}} \left(\pi_s-\widehat{\pi}_s\right) \mathbb{P}_{X|S=s}\left(\widehat{p}_k(X,S) > \widehat{\lambda}+ \dfrac{\widehat{\gamma}_{k,s}}{\widehat{\pi}_s}\right)\right| \leq 
K\left|\mathcal{S}\right| \max_{s \in \mathcal{S}} \left|\widehat{\pi}_s-\pi_s\right|.
\end{equation}
For the second term, conditional on $D_n$, since $\widehat{p}_k$ satisfies similar assumption as Assumption~\ref{assum:continuous}, we observe that
\begin{multline*}
\sum_{k \in [K]} \sumover{s \in \mathcal{S}} \widehat{\pi}_s \mathbb{P}_{X|S=s}\left(\widehat{p}_k(X,S) > \widehat{\lambda}+ \dfrac{\widehat{\gamma}_{k,s}}{\widehat{\pi}_s}\right) - \beta = \\
\sum_{k \in [K]} \sumover{s \in \mathcal{S}} \pi_s\left( \mathbb{P}_{X|S=s}\left(\widehat{p}_k(X,S) > \widehat{\lambda}+ \dfrac{\widehat{\gamma}_{k,s}}{\widehat{\pi}_s}\right) - \widehat{P}_{X|S=s}\left(\widehat{p}_k(X,S) > \widehat{\lambda}+ \dfrac{\widehat{\gamma}_{k,s}}{\widehat{\pi}_s}\right) \right)  \\ + \sum_{k \in [K]} \sum_{s \in \mathcal{S}} \widehat{P}_{X|S=s}\left(\widehat{p}_k(X,S) > \widehat{\lambda}+ \dfrac{\widehat{\gamma}_{k,s}}{\widehat{\pi}_s}\right).
\end{multline*}
Now, we observe that from Lemma~\ref{lem:subgradient_computation}, since $\widehat{H}$ is convex and coercive,   $(\widehat{\lambda}, \widehat{\gamma}$ is a global minimizer that satisfies the first order condition. Therefore, $0 \in \partial_{\lambda}(\lambda^*, \gamma^*)$. Hence, there exists $\mu \in [0,1]$ such that
\begin{equation*}
 \sum_{k = 1}^K\sum_{s \in \mathcal{S}} \dfrac{\widehat{\pi}_s}{N_s}\sum_{i \in \mathcal{D}_{N_s}} \1{\widehat{p}_{k}(X_i,s) > \lambda+\dfrac{\gamma_{k,s}}{\widehat{\pi}_s}} = \beta -  \mu \sum_{k = 1}^K\sum_{s \in \mathcal{S}} \dfrac{\widehat{\pi}_s}{N_s}\sum_{i \in \mathcal{D}_{N_s}} \1{\widehat{p}_{k}(X_i,s) = \lambda+\dfrac{\gamma_{k,s}}{\widehat{\pi}_s}}.
\end{equation*}
Then from Equation~\ref{eq:eqDecompSize1} and~\ref{eq:eqDecompSize2}, we deduce that on the event $\left\{N_s \geq 1\right\}$
\begin{multline*}
\left|\mathcal{T}(\widehat{\Gamma}) -\beta \right| \leq 
K\left|\mathcal{S}\right| \max_{s \in \mathcal{S}} \left|\widehat{\pi}_s-\pi_s\right| + \\
\sum_{k \in [K]} \sumover{s \in \mathcal{S}} \pi_s\left| \mathbb{P}_{X|S=s}\left(\widehat{p}_k(X,S) > \widehat{\lambda}+ \dfrac{\widehat{\gamma}_{k,s}}{\widehat{\pi}_s}\right) - \widehat{P}_{X|S=s}\left(\widehat{p}_k(X,S) > \widehat{\lambda}+ \dfrac{\widehat{\gamma}_{k,s}}{\widehat{\pi}_s}\right)\right| \\
\sum_{k = 1}^K\sum_{s \in \mathcal{S}} \dfrac{\widehat{\pi}_s}{N_s}\sum_{i \in \mathcal{D}_{N_s}} \1{\widehat{p}_{k}(X_i,s) = \lambda+\dfrac{\gamma_{k,s}}{\widehat{\pi}_s}}.
\end{multline*}
Therefore, from Lemma~\ref{lem:pigeonHole} and since $\pi_s \in (0,1)$, we deduce that {\it a.s.}
\begin{equation}
\label{eq:eqDecompSize3}
\left|\mathcal{T}(\widehat{\Gamma})-\beta \right| \leq
K\left|\mathcal{S}\right| \max_{s \in \mathcal{S}} \left|\widehat{\pi}_s-\pi_s\right| + \sum_{k,s} \sup_{t \in [0,1]} \left|\widehat{F}_{k,s}(t) -\widehat{F}_{k,s}(t)\right|
 +\dfrac{C}{\min_{s  \in \mathcal{S}} N_s}.
\end{equation}
Applying Lemma~\ref{lem:DKW}, we have that conditional $D_n$ and $N_s$, on the event $\left\{N_s \geq 1\right\}$, we have that 
\begin{equation*}
\mathbb{E}\left[\sup_{t \in [0,1]} \left|\widehat{F}_{k,s}(t) -\widehat{F}_{k,s}(t)\right|\right] \leq \dfrac{C}{\sqrt{\min_{s \in \mathcal{S}} N_s}}.
\end{equation*}
Therefore, from the above inequality and Equation~\ref{eq:eqDecompSize3}, we deduce that
\begin{multline*}
\mathbb{E}\left[\left|\mathcal{T}(\widehat{\Gamma})-\beta \right|\right]  =   \mathbb{E}\left[\left|\mathcal{T}(\widehat{\Gamma})-\beta \right|\1{N_s \geq 1}+\1{N_s=0}\right]
\leq K\left|\mathcal{S}\right| \max_{s \in \mathcal{S}} \left(\mathbb{E}\left[\left|\widehat{\pi}_s-\pi_s\right|\right]+ \mathbb{E}\left[\dfrac{C \1{N_{{\rm min}} \geq 1}}{\sqrt{N_{\min}}}\right] \right)
\\+ \mathbb{E}\left[\dfrac{C}{N_{\min}}\right] + 2K \mathbb{P}\left(N_{\min} = 0\right). 
\end{multline*}
Let us deal now with the proof of the unfairness bound that follows the same lines than for expected size. Let $k \in [K], s,s' \in \mathcal{S}$.
The following hold
\begin{multline}
\label{eq:eqUnfairProof1}
\left|\mathbb{P}_{X|S=s}\left(\widehat{p}_{k}(X,S) > \widehat{\lambda} + \dfrac{\widehat{\gamma}_{k,s}}{\widehat{\pi}_s}\right) - \mathbb{P}_{X|S=s'}\left(\widehat{p}_{k}(X,S) > \widehat{\lambda} + \dfrac{\widehat{\gamma}_{k,s'}}{\widehat{\pi}_{s'}}\right)\right|
\leq \sup_{t \in \mathbb{R}} \left|F_{k,s}(t)-\widehat{F}_{k,s}(t)\right| \\ + \sup_{t \in \mathbb{R}} \left|F_{k,s'}(t)-\widehat{F}_{k,s'}(t)\right| + \left|\widehat{\mathbb{P}}_{X|S=s}\left(\widehat{p}_{k}(X,S) > \widehat{\lambda} + \dfrac{\widehat{\gamma}_{k,s}}{\widehat{\pi}_s}\right)    - \widehat{\mathbb{P}}_{X|S=s'}\left(\widehat{p}_{k}(X,S) > \widehat{\lambda} + \dfrac{\widehat{\gamma}_{k,s'}}{\widehat{\pi}_{s'}}\right) \right|.
\end{multline}
We consider the function $\widehat{H}$ defined in Lemma~\ref{lem:subgradient_computation}. Since, we minimize this function in $\gamma$ under the constraints that $\sum_{s \in \mathcal{S}} \gamma_{k,s} =0$ for each $k \in [K]$.  we deduce from the KKT conditions that for each $k \in [K], s \in \mathcal{S}$ there exists $\widehat{\nu}_{k \in [K]} \in \mathbb{R}^{K}$ such that
\begin{equation*}
0 \in  \partial_{\gamma_{k,s}} \widehat{H}(\widehat{\lambda}, \widehat{\gamma}) + \widehat{\nu}_k, \;\; 
 {\rm with} \;\; \sum_{s \in \mathcal{S}} \widehat{\gamma}_{k,s} =0.
\end{equation*}
Therefore, from Lemma~\ref{lem:subgradient_computation}, we obtain that there exists $\sigma_{k,s} \in [0,1]$
such that
\begin{equation*}
0 =   - \dfrac{1}{N_s}\sum_{i \in \mathcal{D}_{N_s}} \1{\widehat{p}_{k}(X_i,s) > \lambda+\dfrac{\gamma_{k,s}}{\widehat{\pi}_s}}  - \sigma_{k,s} \dfrac{1}{N_s}\sum_{i \in \mathcal{D}_{N_s}} \1{\widehat{p}_{k}(X_i,s) = \lambda+\dfrac{\gamma_{k,s}}{\widehat{\pi}_s}} + \widehat{\nu}_{k},  
\end{equation*}
that implies that for each $s \neq s' \in \mathcal{S}$
\begin{multline*}
\dfrac{1}{N_s}\sum_{i \in \mathcal{D}_{N_s}} \1{\widehat{p}_{k}(X_i,s) > \lambda+\dfrac{\gamma_{k,s}}{\widehat{\pi}_s}} - \dfrac{1}{N_{s'}}\sum_{i \in \mathcal{D}_{N_{s'}}} \1{\widehat{p}_{k}(X_i,{s'}) > \lambda+\dfrac{\gamma_{k,{s'}}}{\widehat{\pi}_{s'}}}   = \\  \sigma_{k,{s'}} \dfrac{1}{N_{s'}}\sum_{i \in \mathcal{D}_{N_{s'}}} \1{\widehat{p}_{k}(X_i,{s'}) = \lambda+\dfrac{\gamma_{k,{s'}}}{\widehat{\pi}_{s'}}} - \sigma_{k,{s}} \dfrac{1}{N_s}\sum_{i \in \mathcal{D}_{N_s}} \1{\widehat{p}_{k}(X_i,s) = \lambda+\dfrac{\gamma_{k,s}}{\widehat{\pi}_s}} .
\end{multline*}
From the above inequality, we then deduce thanks to Lemma~\ref{lem:pigeonHole}
\begin{equation*}
\left|\widehat{\mathbb{P}}_{X|S=s}\left(\widehat{p}_{k}(X,S) > \widehat{\lambda} + \dfrac{\widehat{\gamma}_{k,s}}{\widehat{\pi}_s}\right)    - \widehat{\mathbb{P}}_{X|S=s'}\left(\widehat{p}_{k}(X,S) > \widehat{\lambda} + \dfrac{\widehat{\gamma}_{k,s'}}{\widehat{\pi}_{s'}}\right) \right| \leq 
\dfrac{C}{N_{\min}}.
\end{equation*}
Combining the above inequality together with Equation~\ref{eq:eqUnfairProof1}, Lemma~\ref{lem:BinomialInverse}, and Lemma~\ref{lem:DKW}, we easily obtain that for each $k, s, s'$
\begin{equation*}
\mathbb{E}\left[\left|\mathbb{P}_{X|S=s}\left(\widehat{p}_{k}(X,S) > \widehat{\lambda} + \dfrac{\widehat{\gamma}_{k,s}}{\widehat{\pi}_s}\right) - \mathbb{P}_{X|S=s'}\left(\widehat{p}_{k}(X,S) > \widehat{\lambda} + \dfrac{\widehat{\gamma}_{k,s'}}{\widehat{\pi}_{s'}}\right)\right|\right]   \leq \sqrt{\dfrac{C_{k,\mathcal{S}}}{N}},  
\end{equation*}
that yields the desired result.
\end{proof}

\begin{proof}[Proof of Theorem~\ref{thm:excess risk}]

For each $(\lambda, \gamma) \in \Delta$, we consider the predictor $\Gamma^*_{\lambda, \gamma}$ defined as 
\begin{equation*}
\Gamma^{*}_{\lambda, \gamma}(x,s) = \left\{k \in [K], \;\; p_k(x,s) \geq \lambda + \dfrac{\gamma_{k,s}}{\pi_{s}}\right\}.    
\end{equation*}
Note that using similar arguments as in the proof of Proposition~\ref{prop:charctOptimRisk}, we can show that the predictor $\Gamma^*$ is optimal with respect to $\mathcal{R}_{\lambda, \gamma}$ defined for a predictor $\Gamma$ by
\begin{equation*}
\mathcal{R}_{\lambda, \gamma}(\Gamma) = R(\Gamma) + \lambda \left(\E[X]{\left\lvert \Gamma(X,S) \right\rvert} - \beta\right) 
        + \sumover[K]{k=1} \sumover{s \in \Scal} \gamma_{k,s} \proba[X \lvert S = s]{k \in \Gamma(X,s)}\enspace.
\end{equation*}
Hence
\begin{equation*}
\Gamma^{*}_{\lambda, \gamma} \in \arg\min_{\Gamma} \mathcal{R}_{\lambda, \gamma} (\Gamma).    
\end{equation*}
Besides similarly to Proposition~\ref{prop:charctOptimRisk}, we have that 
\begin{multline}
\label{eq:eqRiskBound1}
\mathcal{R}_{\lambda, \gamma}\left(\Gamma^{*}_{\lambda, \gamma}\right) -  \mathcal{R}_{\lambda, \gamma} (\Gamma) =   \\
\sumover[K]{k=1} \ \sumover{s \in \Scal} \  \E[X|S=s]{\1{k \in \Gamma(X,s) \Delta \Gamma^*_{\lambda, \gamma}(X,s)} \left\lvert \pi_s\left(p_k(X,s) - \lambda\right) - \gamma_{k,s} \right\rvert} \enspace,
\end{multline}
Finally, we recall that the optimal parameters satisfy
\begin{equation*}
(\lambda^{*}, \gamma^{*}) \in \arg\min_{(\lambda, \gamma) \in \Delta}     \mathcal{R}_{\lambda, \gamma} (\Gamma^*_{\lambda, \gamma}).
\end{equation*}
Now, we start with the following decomposition
\begin{multline}
\label{eq:eqRiskBound2}
            \Rcal_{\lambda^*, \gamma^*}(\widehat{\Gamma}) - \Rcal_{\lambda^*, \alpha^*}(\Gamma^*_{\beta}) = \left(\Rcal_{\lambda^*, \gamma^*}(\widehat{\Gamma}) - \Rcal_{\widehat{\lambda}, \widehat{\gamma}}(\widehat{\Gamma})\right)
            \\ + \left(\Rcal_{\widehat{\lambda}, \widehat{\gamma}}(\widehat{\Gamma}) - \Rcal_{\widehat{\lambda}, \widehat{\gamma}}(\Gamma^*_{\widehat{\lambda}, \widehat{\gamma}}) \right)
             + \left(\Rcal_{\widehat{\lambda}, \widehat{\gamma}}(\Gamma^*_{\widehat{\lambda}, \widehat{\gamma}}) - \Rcal_{\lambda^*, \gamma^*}(\Gamma^*_{\lambda^*, \gamma^*}) \right)
        \end{multline}
        We now act on each of the three terms. For the first term in the {\it r.h.s.}
        of the above equation we observe that
        since for each $k \in [K]$, $\sum_{s} \widehat{\gamma}_{k,s} = \sum_{s} \gamma_{k,s}^* = 0 $, we have that
        \begin{align*}
                \Rcal_{\lambda^*, \gamma^*}(\widehat{\Gamma}) - \Rcal_{\widehat{\lambda}, \widehat{\gamma}}(\widehat{\Gamma}) &= R(\widehat{\Gamma}_{\widehat{\lambda}, \widehat{\gamma}}) - R(\widehat{\Gamma}_{\widehat{\lambda}, \widehat{\gamma}})\\ 
                &\quad + (\lambda^* - \widehat{\lambda}) \left(\E[X,S]{\size{\widehat{\Gamma}_{\widehat{\lambda}, \widehat{\gamma}}(X,S)}} - \beta \right)\\
                &\quad + \sumover{s \in \Scal}\sumover[K]{k=1} (\gamma_{k,s}^* - \widehat{\gamma}_{k,s})\left(\proba[X|S=s]{k \in \widehat{\Gamma}_{\widehat{\lambda}, \widehat{\gamma}}(X,s)} - \proba[X|S=1]{k \in \widehat{\Gamma}_{\widehat{\lambda}, \widehat{\gamma}}(X,S)}\right).\\
            \end{align*}
            Therefore, from Theorem~\ref{thm:Unfairness bound}, since parameters $\lambda^*, \widehat{\lambda}, \gamma^*$, and $\widehat{\gamma}$ are bounded, we deduce that
            \begin{equation}
            \label{eq:boundFirstTerm}
                \Rcal_{\lambda^*, \gamma^*}(\widehat{\Gamma}) - \Rcal_{\widehat{\lambda}, \widehat{\gamma}}(\widehat{\Gamma}) \leq 
                C_{K, \Scal}\sqrt{\dfrac{1}{N}}.
            \end{equation}
            For the second term, we use the characterization of $(\lambda^*, \gamma^*)$ and then observe that
            \begin{equation}
            \label{eq:boundSecondTermByZero}
                \Rcal_{\widehat{\lambda}, \widehat{\gamma}}(\widehat{\Gamma}) - \Rcal_{\widehat{\lambda}, \widehat{\gamma}}(\Gamma^*_{\widehat{\lambda}, \widehat{\gamma}}) \leq 0
            \end{equation}
            
            Finally, we consider the last term in the {\it r.h.s.} of Equation~\eqref{eq:eqRiskBound2}.
            Using Equation~\eqref{eq:eqRiskBound1}, we deduce
            \begin{multline}
            \label{eq:eqRiskBound3}
                {\Rcal_{\widehat{\lambda}, \widehat{\gamma}}(\widehat{\Gamma}_{\widehat{\lambda}, \widehat{\gamma}}) - \Rcal_{\widehat{\lambda}, \widehat{\gamma}}(\Gamma^*_{\widehat{\lambda}, \widehat{\gamma}})}\\ = {\sumover[K]{k=1}\sumover{s \in \Scal} \E[X|S=s]{\1{k \in \widehat{\Gamma}_{\widehat{\lambda}, \widehat{\gamma}}(X,s) \Delta \Gamma^*_{\widehat{\lambda}, \widehat{\gamma}}(X,s)}\left\lvert \pi_s\left(p_k(X,s) - \widehat{\lambda} \right) - \widehat{\gamma}_{k,s} \right\rvert}} 
            \end{multline}
            
            We observe that 
            $k \in \widehat{\Gamma}_{\widehat{\lambda}, \widehat{\gamma}}(X,s) \Delta \Gamma^*_{\widehat{\lambda}, \widehat{\gamma}}(X,s)$ implies
            \begin{eqnarray*}
            \left\lvert \pi_s\left(p_k(X,s) - \widehat{\lambda} \right) - \widehat{\gamma}_{k,s} \right\rvert & \leq & \left\lvert \pi_s\left(p_k(X,s) - \widehat{\lambda} \right) - \widehat{\gamma}_{k,s} -\widehat{\pi}_s\left(\widehat{p}_k(X,S)-\widehat{\lambda}\right)+ \widehat{\gamma}_{k,s}\right\rvert\\
            & \leq & \left|\pi_s(p_k(X,s) - \widehat{p}_k(X,s)) +(\widehat{p}_k(X,s)-\widehat{\lambda})\left(\pi_s-\widehat{\pi}_s\right)\right|.
            \end{eqnarray*}
            Therefore, from Equation~\ref{eq:eqRiskBound3}, since $\widehat{p}_k$, and
            $\widehat{\lambda}$ are bounded, we deduce that
            \begin{multline}
             \label{eq:boundThirdTerm}
            {\Rcal_{\widehat{\lambda}, \widehat{\gamma}}(\widehat{\Gamma}_{\widehat{\lambda}, \widehat{\gamma}}) - \Rcal_{\widehat{\lambda}, \widehat{\gamma}}(\Gamma^*_{\widehat{\lambda}, \widehat{\gamma}})}  \leq
            \sumover[K]{k=1}\sumover{s \in \Scal} \pi_s \E[X|S=s]{\left|\widehat{p}_k(X,s)-p_k(X,s)\right|} + C_K \sum_{s \in \mathcal{S}} \mathbb{E}\left[\left|\widehat{\pi}_s-\pi_s\right|\right]\\
             \leq C_K\left(\underset{s\in\Scal}{\max}\norm[\infty, \mathbb{P}_{X|S=s}]{\widehat{p}- p} + \sqrt{\dfrac{1}{N}}\right)
           \end{multline}
            
        Combining the results from \cref{eq:boundFirstTerm,eq:boundSecondTermByZero,eq:boundThirdTerm}, we obtain
        \begin{align*}
            \Rcal_{\lambda^*, \gamma^*}(\widehat{\Gamma}_{\widehat{\lambda}, \widehat{\gamma}}) - \Rcal_{\lambda^*, \gamma^*}(\Gamma^*_{\lambda^*, \gamma^*}) &\leq C_{K,\Scal}\frac{1}{\sqrt{N}} + C_K\underset{s\in\Scal}{\max}\norm[\infty, \mathbb{P}_{X|S=s}]{\widehat{p}- p},
        \end{align*}
        which yields the desired result.
    \end{proof}

\section{Proof of Section~\ref{sec:twoStepProc}}

 First of all, applying Lemma~\ref{lem:1_over_n_divergence}, we have, almost surely, that for each $u \in (0,1)$, and $\beta \in (0,K)$
 \begin{equation}
 \label{eq:eqPartD0}
0 \leq u -\widehat{\overline{F}}_{k,s}(\widehat{\overline{F}}^{-1}_{k,s} (u)) \leq \dfrac{1}{N_s}, \;\; {\rm and} \;\; 
0 \leq \beta - \widehat{G}(\widehat{G}^{-1}(\beta)) \leq \dfrac{K}{N}.
 \end{equation}

    \begin{proof}[Proof of Theorem~\ref{thm:two-stepMethodBounds}]
        We start the proof the unfairness bound and then establish the bound on the expected size.
        \paragraph*{Unfairness bound. }
        For each $k \in [K], s \in \mathcal{S}$, we introduce
     $\widehat{\beta}_k = \widehat{\overline{F}}_k(\widehat{G}^{-1}(\beta))$, $\widehat{\delta}_{k,s} = \widehat{\overline{F}}_{k,s}^{-1}\left(\widehat{\beta}_k\right)$ ,
        and $\widehat{h}_{k,s} = \widehat{p}_k(X,s) - \widehat{\delta}_{k,s}$.
        We have
        \begin{align*}
            \mathcal{U}(\widehat{\Gamma}) &= \underset{k,s,s'}{\max}\Big(\big\lvert \proba[X|S=s]{\widehat{h}_{k,s} \geq 0} - \probahat[X|S=s]{\widehat{h}_{k,s} \geq 0} + \probahat[X|S=s]{\widehat{h}_{k,s} \geq 0}\\& - \probahat[X|S=s']{\widehat{h}_{k,s'} \geq 0} + \probahat[X|S=s']{\widehat{h}_{k,s'} \geq 0} - \proba[X|S=s']{\widehat{h}_{k,s'} \geq 0} \big\rvert\Big)\\
        \end{align*}
        Then, we deduce that
        \begin{multline}
        \label{eq:eqPartD1}
            \mathcal{U}(\widehat{\Gamma}) \leq 
            2 \underset{k,s}{\max}\left(\left\lvert \proba[X|S=s]{\widehat{h}_{k,s} \geq 0} - \probahat[X|S=s]{\widehat{h}_{k,s} \geq 0}\right\rvert\right) \\ 
            + \underset{k,s,s'}{\max}\left(\left\lvert \probahat[X|S=s]{\widehat{h}_{k,s} \geq 0}- \probahat[X|S=s']{\widehat{h}_{k,s'} \geq 0}\right\rvert\right)
            \end{multline}
        Noting that for 
        \begin{multline*}
            \underset{k,s}{\max}\left(\left\lvert \proba[X|S=s]{\widehat{h}_{k,s} \geq 0} - \probahat[X|S=s]{\widehat{h}_{k,s} \geq 0}\right\rvert\right) \\ \leq \underset{k\in [K]}{\sum}\underset{t}{\sup}\left\lvert \proba[X|S=s]{\widehat{p}_k(X,s) > t} - \probahat[X|S=s]{\widehat{p}_k(X,s) > t}\right\rvert.
        \end{multline*}
        Similarly to the proof of Theorem~\ref{thm:Unfairness bound}, using Lemma~\ref{lem:DKW} conditionally on $D_n$ and $N_s$ and then by integrating over $D_n$ and $N_s$, thanks to Lemma~\ref{lem:BinomialInverse}, we have that there exists $C > 0$ a constant such that:
        \begin{equation*}
            \E{\underset{k,s}{\max}\left(\left\lvert \proba[X|S=s]{\widehat{h}_{k,s} \geq 0} - \probahat[X|S=s]{\widehat{h}_{k,s} \geq 0}\right\rvert\right)} \leq C\frac{K}{\sqrt{N}}.
        \end{equation*}
        
        Furthermore, thanks to Lemma~\ref{lem:1_over_n_divergence},
         since $\probahat[X|S=s]{\widehat{h}_{k,s} \geq 0} = \widehat{\overline{F}}_{k,s}(\widehat{\delta}_{k,s})$. We can write
        \begin{align*}
            \underset{k,s,s'}{\max}\left(\left\lvert \probahat[X|S=s]{\widehat{h}_{k,s} \geq 0}- \probahat[X|S=s']{\widehat{h}_{k,s'} \geq 0}\right\rvert\right) &= \underset{k,s,s'}{\max}\left(\left\lvert \widehat{\overline{F}}_{k,s}(\widehat{\delta}_{k,s}) - \widehat{\overline{F}}_{k,s'}(\widehat{\delta}_{k,s'})\right\rvert\right)\\
            &\leq \underset{k,s,s'}{\max}\left(\left\lvert \widehat{\beta}_k + \frac{1}{N_s} - \widehat{\beta}_k + \frac{1}{N_{s'}}\right\rvert\right)\\
            &\leq \underset{k,s,s'}{\max}\left(\left\lvert \frac{1}{N_s} + \frac{1}{N_{s'}}\right\rvert\right)\\
            &\leq \frac{2}{\underset{s}{\min}N_s}.
        \end{align*}
        
        In view of Equation~\ref{eq:eqPartD1}, applying again Lemma~\ref{lem:BinomialInverse}, we can now combine all the terms
        \begin{align*}
            \E{\mathcal{U}(\Gamma)} &\leq \frac{C_{K,\Scal}}{\sqrt{N}}.
        \end{align*}

      \paragraph*{Size constraint violation.}
      We can write:
        \begin{align*}
            \left\lvert \E[X,S]{\size{\widehat{\Gamma}(X,S)}} - \beta \right\rvert &= \left\lvert \sumover[K]{k=1}\sumover{s \in \Scal} \pi_s\proba[X|S=s]{k \in \widehat{\Gamma}(X,s)} - \beta \right\rvert\\
            &= \left\lvert \sumover[K]{k=1}\sumover{s \in \Scal} \pi_s\left(\overline{F}_{k,s}\left(\widehat{\delta}_{k,s}\right)\right) - \beta\right\rvert\\
            &= \left\lvert\sumover[K]{k=1}\sumover{s \in \Scal} \pi_s\left(\overline{F}_{k,s}\left(\widehat{\delta}_{k,s}\right) - \widehat{\overline{F}}_{k,s}\left(\widehat{\delta}_{k,s}\right) + \widehat{\overline{F}}_{k,s}\left(\widehat{\delta}_{k,s}\right)\right) - \beta \right\rvert
        \end{align*}
        Then we deduce that
        \begin{equation*}
             \left\lvert \E[X,S]{\size{\widehat{\Gamma}(X,S)}} - \beta \right\rvert \leq 
             \left\lvert\sumover[K]{k=1}\sumover{s \in \Scal} \pi_s\left(\overline{F}_{k,s}\left(\widehat{\delta}_{k,s}\right) - \widehat{\overline{F}}_{k,s}\left(\widehat{\delta}_{k,s}\right)\right) \right\rvert + \left\lvert \sumover[K]{k=1}\sumover{s \in \Scal} \pi_s \widehat{\overline{F}}_{k,s}\left(\widehat{\delta}_{k,s}\right) - \beta \right\rvert.
        \end{equation*}
        To control the first term in the {\it r.h.s.} of the above equation we use Lemma~\ref{lem:DKW}. For the second term, we observe that
        \begin{multline*}
            \left\lvert \sumover[K]{k=1}\sumover{s \in \Scal} \pi_s \widehat{\overline{F}}_{k,s}\left(\widehat{\delta}_{k,s}\right) - \beta \right\rvert 
            \\
            \leq \left\lvert \sumover[K]{k=1}\sumover{s \in \Scal} \pi_s \widehat{\overline{F}}_{k,s}\left(\widehat{\delta}_{k,s}\right) - \sum_{k \in [K]}\sum_{s \in \mathcal{S}} \pi_s \widehat{\overline{F}}_k(\widehat{G}^{-1}(\beta)) \right\rvert + \left\lvert \sum_{k \in [K]}\sum_{s \in \mathcal{S}} \pi_s \widehat{\overline{F}}_k(\widehat{G}^{-1}(\beta)) - \beta \right\rvert \enspace.
        \end{multline*}
        From Equation~\ref{eq:eqPartD0}, we then deduce that almost surely
        \begin{equation*}
            \left\lvert \sumover[K]{k=1}\sumover{s \in \Scal} \pi_s \widehat{\overline{F}}_{k,s}\left(\widehat{\delta}_{k,s}\right) - \beta \right\rvert \leq
            \dfrac{1}{N_s} + \dfrac{K}{N}.
        \end{equation*}
        Therefore, applying again Lemma~\ref{lem:BinomialInverse} we deduce the desired result.
    \end{proof}
        

\end{document}